\documentclass{article}

\usepackage[final]{ewrl_2026}

\usepackage[utf8]{inputenc}
\usepackage[T1]{fontenc}
\usepackage{hyperref}
\usepackage{url}
\usepackage{xcolor}

\usepackage{amsmath}
\usepackage{amssymb}
\usepackage{mathtools}
\usepackage{mathrsfs}
\usepackage{amsfonts}
\usepackage{bm}

\usepackage{graphicx}
\usepackage{subcaption}
\usepackage[space]{grffile}
\usepackage{booktabs}
\usepackage{wrapfig}

\usepackage{nicefrac}
\usepackage{microtype}
\usepackage{doi}
\usepackage{physics}

\usepackage{algorithm}
\usepackage{algpseudocode}
\usepackage{forest}
\usepackage{tikz-cd}

\newenvironment{proof}{\paragraph{Proof:}}{\hfill$\square$}

\newtheorem{theorem}{Theorem}
\newtheorem{proposition}[theorem]{Proposition}
\newtheorem{lemma}[theorem]{Lemma}
\newtheorem{corollary}[theorem]{Corollary}

\newtheorem{assumption}[theorem]{Assumption}
\newtheorem{remark}[theorem]{Remark}

\title{The Price of Decentralization in Top-$K$ Arm Identification}

\author{%
  Larissa Xu \And
  Jasmine Nguyen \And
  William Chang \\
  Department of Applied Mathematics, University of California, Los Angeles, USA \\
  \texttt{\{xuzhiyun004119, jasmineng306, chang314\}@g.ucla.edu}
}

\begin{document}

\maketitle

\begin{abstract}
Cooperative teams often need to agree on the best \emph{few} options rather than simply accumulate reward, and they must do so while each member sees only a fragment of the team's collective experience. We study this as \emph{top-$K$ joint-arm identification} in multi-agent multi-armed bandits: at every round $M$ agents simultaneously choose individual actions that compose a joint arm, and the team must ultimately return the $K$ joint arms of highest mean reward. The difficulty is that an agent may not observe the actions of others, their rewards, or either. We treat three observability regimes---(A) shared rewards with hidden actions, (B) observed actions with private rewards, and (C) full asymmetry---and design communication-free elimination algorithms (\texttt{UCB-Intervals}) that reconstruct implicit coordination from whatever signal each regime leaves intact: a shared arm ordering in~(A), observable deviations in~(B), and enlarged confidence radii under~(C). We give matching analyses in both the fixed-budget and fixed-confidence objectives, then fold all three regimes into a single meta-guarantee indexed by a multiplicity $c$ and a consensus factor $\rho$. Our central result is quantitative rather than merely algorithmic: change-of-measure lower bounds show that shared-reward identification is optimal up to one universal logarithmic factor, and that the \emph{entire} statistical price of removing communication is a multiplicative $\rho^2$ in sample complexity---a fixed $4\times$ penalty under full asymmetry. The resulting stopping time scales as $O\!\left(\sum_{\bm{a}} \log(A^M/\delta)/\Delta_{\bm{a}}^{2}\right)$ and the fixed-budget error as $\exp(-\Theta(T/H_1))$, with the dependence on the joint-action count $A^M$ shown to be unavoidable.
\end{abstract}

\section{Introduction}

The problem of identifying the best arm in multi-armed bandit (MAB) problems is a well-studied area in machine learning. Introduced by \cite{audibert2010best}, the UCB-E and Successive Rejects algorithms address this under the fixed budget setting. \cite{bubeck2013multiple} extended Successive Rejects to identify multiple best arms, and \cite{even2006action} introduced the fixed-confidence setting where the goal is to identify the best arm with probability at least $1-\delta$ while minimizing sample complexity.

Many cooperative decisions are made by teams that must agree on a small set of joint configurations while each member sees only its own slice of the outcome. In decentralized scientific experimentation---drug-combination screening, materials discovery, or large-scale A/B testing across independently operated services---several groups each control one component of a joint intervention and observe only their own experimental readout, because of privacy, bandwidth, or institutional boundaries. Federated experimentation and privacy-preserving multi-site trials share the same shape: a common underlying effect is measured independently at each site, so the sites obtain independent noisy samples of the same joint configuration rather than a shared one. These applications motivate two features of our model that might otherwise look like modeling choices. First, the \emph{joint action space is unstructured} in the worst case: nothing forces the reward of a joint configuration to be a smooth function of its components (a drug pair can be toxic even when each drug is safe), so a framework that presumes structure would be unsound where the structure is absent. We therefore treat the general (structure-free) case as the baseline, and show separately in Section~\ref{sec:additive} that when the reward \emph{does} decompose additively across players the exponential joint-arm count collapses to a polynomial one---so structure is exploited whenever it is present, rather than assumed. Second, \emph{reward asymmetry}---players receiving independent samples of the same arm---is exactly the federated/multi-site measurement model, not an artifact. Problems of identifying the best arms under limited sampling have been widely studied \cite{audibert2010best, gabillon2012best, kaufmann2016complexity, lattimore2020bandit}, but extending these ideas to decentralized environments with multiple decision-makers and asymmetric information introduces significant new challenges: the difficulty is no longer only statistical allocation, but coordination and consensus under partial observability.

\cite{gabillon2012best} unified both settings by refining the UCB-E algorithm. We improve on their approach by constructing UCB intervals centered at the empirical mean and extend to multiple agents operating simultaneously. Recent works such as \cite{boursier2019sic} explored collaborative algorithms leveraging collision-based implicit communication, while \cite{karpov2023communication} proposed communication-minimizing top-$m$ identification. However, these approaches rely on synchronized players, structured feedback, or centralized coordination, limiting applicability in fully decentralized environments. We give a more detailed overview in appendix \ref{sec:related}.

\paragraph{Our Contributions.}
We introduce the multiplayer framework for top-$K$ arm identification under two forms of information asymmetry \cite{chang2022online, chang2023optimal}: asymmetry in actions and asymmetry in rewards. We propose decentralized algorithms extending classical elimination-based bandit methods to the multi-agent setting without communication, enabling agents to implicitly coordinate despite limited information. We provide theoretical guarantees for both fixed-budget and fixed-confidence regimes, and validate all algorithms empirically in Section~\ref{sec:experiments}.

\section{Preliminaries}
\label{sec:prelim}

\paragraph{Multi-player joint-action bandits.}
We consider $M$ cooperative players. Player $i$ has an action set of size $A$, and at each round $t$ all players act simultaneously, producing a \emph{joint action} $\bm{a}_t \in \mathcal{A} = [A]^M$ (with $|\mathcal{A}| = A^M$).
We write $\bm{a}[i]$ for the $i$-th component. Players do not communicate during learning but may agree on a strategy (including a shared ordering of $\mathcal{A}$) beforehand.
Each joint arm $\bm{a}$ has an unknown mean reward $\mu_{\bm{a}}$, and all reward noise is $1$-subgaussian (equivalently, unit-variance in the Gaussian case); this normalization is without loss of generality up to rescaling the gaps. Let $\sigma:\{1,\dots,A^M\}\to\mathcal{A}$ order the arms by decreasing mean, so that $\mu_{\sigma(1)}\ge\cdots\ge\mu_{\sigma(A^M)}$; we call $\sigma(1),\dots,\sigma(K)$ the top-$K$ joint arms and write $S^\star=\{\sigma(1),\dots,\sigma(K)\}$ for this set (or $S^\star(\nu)$ when the instance $\nu$ must be made explicit). Throughout, $\sigma$ denotes only this ranking permutation; the sub-Gaussian constant plays no free role, having been fixed to $1$ above.

\paragraph{Gaps and complexity.}
For any $\bm{a}\in\mathcal{A}$ with rank $r(\bm{a}) := \sigma^{-1}(\bm{a})$, define the top-$K$ gap
\begin{equation}\label{eq:gap-def}
\Delta_{\bm{a}}^{\langle K\rangle}
:=
\begin{cases}
\mu_{\bm{a}} - \mu_{\sigma(K+1)}, & r(\bm{a}) \le K,\\
\mu_{\sigma(K)} - \mu_{\bm{a}}, & r(\bm{a}) > K,
\end{cases}
\end{equation}
the smallest gap $\Delta_{\min}:=\min_{\bm a:\,\Delta_{\bm a}^{\langle K\rangle}>0}\Delta_{\bm a}^{\langle K\rangle}$, and the complexity $H_1 := \sum_{\bm{a}} (\Delta_{\bm{a}}^{\langle K\rangle})^{-2}$. When no confusion arises we abbreviate $\Delta_{\bm a}\equiv\Delta_{\bm a}^{\langle K\rangle}$. We also use the sorted-gap complexity $H_2:=\max_{\bm a}\,r(\bm a)\,(\Delta_{\bm a}^{\langle K\rangle})^{-2}$, which satisfies $H_2\le H_1\le H_2\log(A^M)$ \cite{audibert2010best}.
Since identifying the top-$K$ joint actions is at least as hard as the single-player top-$K$ problem over $A^M$ arms, any dependence on the joint action space size---potentially scaling as $K^M$---is information-theoretically unavoidable.

\paragraph{Notation.}
We collect the recurring symbols here for reference. $\widehat{\mu}_{\bm a}^{i}(t)$ is player $i$'s empirical mean of arm $\bm{a}$ after round $t$, and $n_{\bm a}^{i}(t)$ the number of rounds up to $t$ in which player $i$ has sampled $\bm{a}$; when rewards are shared (Problem A) all players hold identical statistics and we drop the superscript, writing $\widehat{\mu}_{\bm a}(t)$ and $n_{\bm a}(t)$. The confidence radius is $\epsilon_{\bm a}^{i}(t)$ (again $\epsilon_{\bm a}(t)$ under shared rewards), and $I_{\bm a}^{i}(t)=[\widehat{\mu}_{\bm a}^{i}(t)-\epsilon_{\bm a}^{i}(t),\,\widehat{\mu}_{\bm a}^{i}(t)+\epsilon_{\bm a}^{i}(t)]$ the induced interval. The algorithm's returned set is $\widehat{S}$ (or $\widehat{S}_i$ for player $i$'s individual output), to be compared against $S^\star$. In the fixed-budget setting $R_T:=\mathbb{P}(\widehat{S}\neq S^\star)$ denotes the error probability after $T$ rounds. We write $\widetilde{O},\widetilde{\Theta}$ for bounds that suppress polylogarithmic factors.

\paragraph{Concentration.}
We use the following standard result throughout (Corollary~5.5 in \cite{lattimore2020bandit}), specialized to our $1$-subgaussian rewards.

\begin{lemma}\label{corollary5.5}
Let $X_1-\mu,\dots,X_n-\mu$ be independent and $1$-subgaussian, and $\widehat{\mu}=\tfrac1n\sum_{k}X_k$. Then for any $\varepsilon\ge0$,
$\mathbb{P}(\widehat{\mu} \ge \mu+\varepsilon) \le \exp(-n\varepsilon^2/2)$
and
$\mathbb{P}(\widehat{\mu} \le \mu-\varepsilon) \le \exp(-n\varepsilon^2/2)$.
\end{lemma}

\section{Problem Setting}
\label{sec:setting}

\paragraph{Reward model.}
Every joint arm $\bm a$ carries a single underlying mean $\mu_{\bm a}$; the three problems differ only in \emph{which realization of that arm each player observes} and \emph{whether the joint action is visible}. When the prescribed joint action at round $t$ is $\bm a$, the reward player $i$ records is $r_t^i(\bm a)=\mu_{\bm a}+\eta_t^i$ with $1$-subgaussian noise $\eta_t^i$. In \textbf{Problem A} the noise is common, $\eta_t^1=\cdots=\eta_t^M$, so all players observe the \emph{same} sample $r_t(\bm a)$; in \textbf{Problems B and C} the $\eta_t^i$ are independent across players, so each player draws an independent i.i.d.\ sample of the same arm. Thus a joint arm has one true value in every problem, and the problems are ordered by how much of the team's information any one player can reconstruct.

\paragraph{Information asymmetry models.}
\textbf{Problem A (action asymmetry):} rewards are shared (common noise) but a player cannot observe other players' actions. \textbf{Problem B (reward asymmetry):} each player observes the realized joint action $\bm a_t$ but receives an independent reward. \textbf{Problem C (full asymmetry):} each player receives an independent reward \emph{and} cannot observe other players' actions.

\paragraph{Why more visibility need not mean an easier problem.}
It is tempting to read Problem B as strictly more informative than Problem A---B additionally reveals the joint action---and therefore to expect better guarantees. The reverse holds, and the reason is that the two problems differ in \emph{reward} information, not action information. In Problem A the reward is a single shared random variable, so the whole team concentrates through \emph{one} statistical event and behaves exactly like a centralized learner. In Problem B each player must estimate every arm from its \emph{own} independent stream, so $M$ concentration events must hold simultaneously; the union bound over these streams is precisely the multiplicative $M$ (fixed budget) and additive $\log M$ (fixed confidence) overhead quantified below. The extra action visibility in B is not wasted---it is exactly what powers the deviation signal that keeps the $M$ streams coordinated---but it only recovers the coordination that shared rewards grant A for free; it cannot undo the loss of a shared reward stream.

\paragraph{Learning objectives.}
In the \emph{fixed budget} setting, players are given horizon $T$ and minimize the error probability $R_T=\mathbb{P}(\widehat{S}\neq S^\star)$ after $T$ rounds. In the \emph{fixed confidence} setting, given $\delta\in(0,1)$, the goal is to identify the top-$K$ arms with probability $\ge 1-\delta$ while minimizing the stopping time $\tau$, with \emph{all} players outputting the correct set at termination.

\paragraph{Single-agent baseline and information-theoretic costs.}
Before analyzing each regime, it is instructive to recall what a single centralized agent can achieve when given direct access to all observations.  For top-$K$ identification over $N$ arms, \cite{gabillon2012best} show that UCB-E achieves fixed-budget error $O(N T \exp(-T/(16H_1)))$ and \cite{chen2015optimal} give a fixed-confidence stopping time of $\tau = O(H_1 \log(N/\delta)/\Delta_{\min}^2)$ (up to constants), with a matching lower bound of $\Omega(H_1 \log(1/\delta)/\Delta_{\min}^2)$.  With $N = A^M$ joint arms our Problem A \emph{exactly recovers} these single-agent rates: shared rewards let all players maintain the same sufficient statistics, so the distributed problem collapses to a centralized one.  Problem B incurs a multiplicative factor of $M$ in the fixed-budget error probability and an additive $\log M$ term inside the fixed-confidence logarithm---the price of $M$ independent concentration events that must all hold simultaneously.  Problem C pays a further $4\times$ penalty in fixed-confidence sample complexity from the enlarged consensus radius, and loses the adaptive allocation benefit of UCB-Intervals in fixed budget: the exponent $-T\Delta^2/(8A^M)$ is weaker than $-T/(32H_1)$ whenever $H_1 \ll A^M/\Delta_{\min}^2$, which occurs whenever some arms have large gaps and can be eliminated early.

\paragraph{UCB-Intervals template.}
Our algorithms maintain confidence intervals $I_{\bm{a}}^{i}(t) = [\widehat{\mu}_{\bm{a}}^{i}(t) - \epsilon_{\bm{a}}^{i}(t),\; \widehat{\mu}_{\bm{a}}^{i}(t) + \epsilon_{\bm{a}}^{i}(t)]$ for each player $i$ and joint arm $\bm{a}$, initialized to $(-\infty,\infty)$.
We say $\bm{a}'$ \emph{dominates} $\bm{a}$ for player $i$ when $\min I_{\bm{a}'}^{i}(t) > \max I_{\bm{a}}^{i}(t)$, and eliminate $\bm{a}$ from the candidate set $\mathcal{D}$ once $K$ arms each dominate it.
The radius $\epsilon_{\bm{a}}^{i}(t)$ is set to $\sqrt{\alpha/n_{\bm{a}}^{i}(t)}$ in the fixed-budget setting (with $\alpha \le T/(16H_1)$ by Lemma~\ref{lem:num_pulls}) and to a time-uniform, $\delta$-dependent expression in the fixed-confidence setting. Under shared rewards (Problem A) the superscript $i$ is dropped, since all players hold identical intervals.

\begin{remark}[On the knowledge of $H_1$]
\label{rem:H1-knowledge}
The constant $\alpha=T/(16H_1)$ uses the unknown complexity $H_1$, the standard caveat of UCB-E-style methods \cite{audibert2010best,gabillon2012best}. Only an \emph{upper bound} is needed for correctness, and the usual doubling schedule over guesses $\widehat H\in\{2,4,8,\dots\}$ removes even that at a constant-factor cost. Appendix~\ref{apx:experiments} reports an $\alpha$ sensitivity study.
\end{remark}

\section{Problem A: Action Asymmetry}
\label{sec:probA}

In Problem A, all players receive the same reward signal but cannot observe each other's actions.
Despite the lack of action observability, a pre-agreed fixed ordering of $\mathcal{A}$ is enough for full implicit coordination: since rewards are shared, every player maintains \emph{identical} empirical means and confidence intervals throughout.
Eliminations are therefore \emph{synchronous}---whenever one player's intervals certify that $\bm{a}$ is dominated by $K$ candidates, all players certify the same at the same round and eliminate simultaneously without any message passing.
The information structure of Problem A is thus equivalent to that of a single centralized agent observing $A^M$ arms, and we should expect---and will confirm---that our bounds match the single-agent UCB-E benchmark of \cite{gabillon2012best}. Proofs of all results in this section are deferred to Appendix~\ref{apx:proofs}.

\begin{algorithm}[H]
\caption{\texttt{UCB-Intervals} for Problem A (Fixed Budget)}
\label{alg:ucb-A-budget}
\begin{algorithmic}[1]
\Require $A^M$, budget $T$, target $K$, $\alpha = T/(16H_1)$
\State $\mathcal{D} \gets \mathcal{A}$; fix an ordering; $I_{\bm{a}} \gets (-\infty,\infty)$ for all $\bm{a}$
\For{$t = 1$ to $T$}
    \State Select next $\bm{a} \in \mathcal{D}$ per ordering
    \While{$\exists\,\mathcal{S}\!\subseteq\!\mathcal{D},\,|\mathcal{S}|\!=\!K : \forall\bm{b}\!\in\!\mathcal{S},\;\min I_{\bm{b}} > \max I_{\bm{a}}$}
        \State $\mathcal{D} \gets \mathcal{D}\setminus\{\bm{a}\}$; select next $\bm{a} \in \mathcal{D}$
    \EndWhile
    \State All players pull component $\bm{a}[i]$; observe shared reward; update $I_{\bm{a}}$
\EndFor
\State \Return $\mathcal{D}$
\end{algorithmic}
\end{algorithm}

\begin{lemma}
\label{lem:num_pulls}
Algorithm~\ref{alg:ucb-A-budget} completes all eliminations within horizon $T$ provided $\alpha \leq T/(16H_1)$.
\end{lemma}

\begin{theorem}[Problem A, Fixed Budget]
\label{thm:algoA-budget}
With $\alpha = T/(16H_1)$ and $\epsilon_{\bm{a}}(t) = \sqrt{\alpha/n_{\bm{a}}(t)}$ (the agent index $i$ is dropped throughout Problem~A, since shared rewards give every player identical statistics $\widehat\mu_{\bm a}(t)$, $n_{\bm a}(t)$, and hence a common radius), Algorithm~\ref{alg:ucb-A-budget} satisfies
\[
R_T \;\le\; A^M T \exp\!\bigl(-\alpha/2\bigr) \;=\; A^M T\exp\!\Bigl(-\frac{T}{32H_1}\Bigr).
\]
\end{theorem}

\begin{theorem}[Problem A, Fixed Confidence]
\label{thm:algoA-confidence}
Using the time-uniform radius $\epsilon_{\bm{a}}(t) = \sqrt{2\log(\pi^2 A^M t^2/(3\delta))\,/\,n_{\bm{a}}(t)}$ (see Algorithm~\ref{alg:ucb-A-confidence} in Appendix~\ref{apx:fc-algos}), the fixed-confidence variant returns the correct top-$K$ with probability $\ge 1-\delta$ and stopping time
\begin{equation}
\label{eq:tau-A}
\tau \;\le\;
\underbrace{K \cdot \frac{2\log(\pi^2 A^M \tau^2/(3\delta))}{\bigl(\Delta_{\sigma(K+1)}^{\langle K\rangle}\bigr)^2}}_{\text{pulls for }K\text{ boundary arms}}
\;+\;
\underbrace{\sum_{\sigma^{-1}(\bm{a})>K}
\frac{2\log(\pi^2 A^M \tau^2/(3\delta))}{\bigl(\Delta_{\bm{a}}^{\langle K\rangle}\bigr)^2}}_{\text{pulls until each sub-optimal arm is eliminated}}.
\end{equation}
\end{theorem}

\paragraph{Discussion.}
Both theorems confirm that Problem A incurs \emph{no multi-player overhead}: the fixed-budget bound $A^M T\exp(-T/(32H_1))$ and the fixed-confidence bound $\tau = O(H_1\log(A^M/\delta)/\Delta_{\min}^2)$ are structurally identical to the single-agent UCB-E guarantee on $A^M$ arms \cite{gabillon2012best}.
The two regimes differ structurally: fixed budget gives an error that vanishes exponentially once $T$ exceeds $O(H_1)$, whereas fixed confidence adaptively concentrates samples on hard arms until a certifiably correct answer is reached, at the cost of knowing $\delta$ in advance.
We show in Section~\ref{sec:optimality} that both bounds are optimal up to the universal exploration log-factor (Corollary~\ref{cor:optimal}).

\section{Problem B: Reward Asymmetry}
\label{sec:probB}

In Problem B, all players observe the realized joint action $\bm{a}_t$, but each receives an independent i.i.d.\ reward drawn from the arm's distribution.
Because rewards are no longer shared, players' empirical means diverge over time: player $i$ builds its own estimate $\widehat{\mu}_{\bm{a}}^i$ based solely on its own reward history, and two players may be ready to eliminate arm $\bm{a}$ at different rounds.
To synchronize these asynchronous eliminations without communication, we use an \emph{observable deviation signal}: when the ordering prescribes arm $\bm{a}$ and player $i$ can already certify that $\bm{a}$ is dominated by $K$ arms in its own intervals, player $i$ intentionally pulls an action \emph{different} from $\bm{a}[i]$.
Because all players observe the realized joint action $\bm{a}_t$, any deviation makes $\bm{a}_t \neq \bm{a}$ visible to everyone, and the group collectively removes $\bm{a}$ from $\mathcal{D}$.
A single bit of observable behavior—one player's deviation—thus serves as a broadcast elimination signal, synchronizing heterogeneous candidate sets without a single explicit message. Proofs for this section are deferred to Appendix~\ref{apx:proofs}.

\begin{algorithm}[H]
\caption{\texttt{UCB-Intervals} for Problem B (Fixed Budget)}
\label{alg:ucb-B-budget}
\begin{algorithmic}[1]
\Require $A^M$, budget $T$, target $K$, $\alpha = T/(16H_1)$
\State $\mathcal{D} \gets \mathcal{A}$; fix an ordering; $I_{\bm{a}}^{i} \gets (-\infty,\infty)$ for all $i, \bm{a}$
\For{$t = 1$ to $T$}
    \State Select next $\bm{a} \in \mathcal{D}$ per ordering
    \For{each player $i$}
        \If{$\exists\,\mathcal{S}\!\subseteq\!\mathcal{D},\,|\mathcal{S}|\!=\!K : \forall\bm{b}\!\in\!\mathcal{S},\;\min I_{\bm{b}}^{i} > \max I_{\bm{a}}^{i}$}
            \State Pull action $\neq \bm{a}[i]$ \hfill\Comment{deviation = elimination signal}
        \Else
            \State Pull $\bm{a}[i]$
        \EndIf
    \EndFor
    \For{each player $i$}
        \State Observe $\bm{a}_t$; if $\bm{a}_t \neq \bm{a}$, set $\mathcal{D} \gets \mathcal{D} \setminus \{\bm{a}\}$
        \State Observe reward; update $I_{\bm{a}_t}^{i}$ with $\epsilon_{\bm{a}}(t) = \sqrt{\alpha/n_{\bm{a}}^i(t)}$
    \EndFor
\EndFor
\State \Return $\mathcal{D}$
\end{algorithmic}
\end{algorithm}

\begin{theorem}[Problem B, Fixed Budget]
\label{thm:algoB-budget}
With $\alpha = T/(16H_1)$, Algorithm~\ref{alg:ucb-B-budget} satisfies
\[
R_T \;\le\; M \cdot A^M T\exp\!\Bigl(-\frac{T}{32H_1}\Bigr).
\]
\end{theorem}

\begin{theorem}[Problem B, Fixed Confidence]
\label{thm:algoB-confidence}
Using the per-player radius
\begin{equation}
\label{eq:eps-B-conf}
\epsilon_{\bm{a}}^{i}(t) = \sqrt{\frac{2\log(\pi^2 A^M M t^2/(3\delta))}{n_{\bm{a}}^{i}(t)}}
\end{equation}
(see Algorithm~\ref{alg:ucb-B-confidence} in Appendix~\ref{apx:fc-algos}), the fixed-confidence variant returns the correct top-$K$ with probability $\ge 1-\delta$ and stopping time
\begin{equation}
\label{eq:tau-B}
\tau \;\le\;
K \cdot \frac{2\log(\pi^2 A^M M \tau^2/(3\delta))}{\bigl(\Delta_{\sigma(K+1)}^{\langle K\rangle}\bigr)^2}
\;+\;
\sum_{\sigma^{-1}(\bm{a})>K}
\frac{2\log(\pi^2 A^M M \tau^2/(3\delta))}{\bigl(\Delta_{\bm{a}}^{\langle K\rangle}\bigr)^2}.
\end{equation}
\end{theorem}

\paragraph{Discussion.}
Comparing Problems A and B reveals the precise cost of independent rewards.
In the \textbf{fixed-budget} setting, the bound for Problem B is exactly $M$ times that of Problem A.
This factor is tight in our analysis: the good event $G = \bigcap_{i,\bm{a},t}\{\mu_{\bm{a}} \in I_{\bm{a}}^i(t)\}$ must hold simultaneously for all $M$ players, and a union bound introduces the $M$ factor.
When $M$ is small (e.g.\ $M = 2$), this overhead is modest, but it grows linearly with the number of players, reflecting a genuine hardness gap between Problems A and B.
In the \textbf{fixed-confidence} setting, the penalty is milder: the $M$ factor appears only inside the logarithm, changing the stopping-time guarantee from $O(H_1\log(A^M\tau^2/(3\delta))/\Delta_{\min}^2)$ to $O(H_1\log(MA^M\tau^2/(3\delta))/\Delta_{\min}^2)$.
The additive overhead is $O(H_1\log M / \Delta_{\min}^2)$---vanishingly small relative to the $O(H_1\log(A^M)/\Delta_{\min}^2) = O(M H_1\log A/\Delta_{\min}^2)$ dominant term.
Intuitively, more confidence is needed in the fixed-budget case because an error by \emph{any} player causes an overall failure, whereas in fixed confidence the algorithm can run longer to achieve the required per-player certainty, paying only a logarithmic surcharge.

\begin{remark}[Cost of the deviation signal]
\label{rem:deviation-cost}
A deviation is emitted only to \emph{eliminate} an arm, and each costs one off-schedule pull. With at most $A^M-K$ eliminations over the run, the total wasted signaling is at most $A^M-K$ pulls---lower-order against the $\Omega(H_1\log(1/\delta))$ pulls already required, and absorbed by the constants in Theorems~\ref{thm:algoB-budget}/\ref{thm:algoB-confidence}.
\end{remark}

\begin{remark}[Robustness of the signal and the explicit-communication baseline]
\label{rem:deviation-robust}
The signal is a one-bit public broadcast (``eliminate the current arm''), so Problem~B matches a baseline spending one explicit bit per elimination, at the wasted-round cost of Remark~\ref{rem:deviation-cost}. The bare signal assumes the joint action is observed exactly; under noisy or partial observation a single mislabeled action could trigger a false elimination. The standard fix is redundancy: repeat the off-schedule action for $\ell$ rounds (or an agreed error-correcting pattern), which tolerates $\lfloor(\ell-1)/2\rfloor$ corrupted observations per signal at an $\ell$-fold signaling cost---still lower-order---and leaves the analysis unchanged. A full treatment of adversarial misobservation is left to future work.
\end{remark}

\section{Problem C: Full Asymmetry}
\label{sec:probC}

In Problem C, players can neither observe each other's actions nor share rewards, so both coordination mechanisms of Problems A and B are unavailable: the fixed-ordering protocol needs the group's joint action to be observable, and the deviation signal needs deviating players to be detectable. We therefore adopt \emph{uniform exploration}: players pre-agree on a round-robin schedule visiting each joint arm $\lfloor T/A^M\rfloor$ times, requiring no online coordination. Each player forms its own estimate $\widehat{\mu}_{\bm{a}}^i$ and independently outputs its empirical top-$K$. Guaranteeing that \emph{all} players output the same correct set then rests on the consensus argument below: with a suitably enlarged radius, any ranking certifiable by one player is certifiable by all. We discuss the resulting decentralized stopping rule in Remark~\ref{rem:C-stopping}; proofs for this section are deferred to Appendix~\ref{apx:proofs}.

\begin{algorithm}[H]
\caption{Uniform Exploration for Problem C (Fixed Budget)}
\label{alg:unif-C-budget}
\begin{algorithmic}[1]
\Require $A^M$, budget $T$, target $K$; run independently by each player $i$
\State $n \gets \lfloor T/A^M \rfloor$; $\widehat{\mu}_{\bm{a}}^{i} \gets 0$ for all $\bm{a}$
\ForAll{$\bm{a} \in \mathcal{A}$ (pre-agreed round-robin order)}
    \For{$t = 1$ to $n$} player $i$ executes $\bm{a}[i]$; observes own reward $r_t^{i}(\bm{a})=\mu_{\bm a}+\eta_t^i$ \EndFor
    \State $\widehat{\mu}_{\bm{a}}^{i} \gets \frac{1}{n}\sum_{t=1}^{n} r_t^{i}(\bm{a})$
\EndFor
\State \Return player $i$'s top-$K$ arms by $\widehat{\mu}_{\bm{a}}^{i}$
\end{algorithmic}
\end{algorithm}

\begin{theorem}[Problem C, Fixed Budget]
\label{thm:algoC-budget}
Algorithm~\ref{alg:unif-C-budget} satisfies
\[
R_T \;\leq\; M\sum_{\bm{a}\in\mathcal{A}} 2\exp\!\left(-\frac{T(\Delta_{\bm{a}}^{\langle K\rangle})^2}{8A^M}\right).
\]
\end{theorem}

\begin{theorem}[Problem C, Fixed Confidence]
\label{thm:algoC-confidence}
Using the enlarged per-player radius
\begin{equation}
\label{eq:eps-C-conf}
\epsilon_{\bm{a}}^{i}(t) = 2\sqrt{\frac{2\log(\pi^2 A^M M t^2/(3\delta))}{n_{\bm{a}}^{i}(t)}}
\end{equation}
(see Algorithm~\ref{alg:unif-C-confidence} in Appendix~\ref{apx:fc-algos}), the fixed-confidence variant returns the correct top-$K$ with probability $\ge 1-\delta$ and stopping time
\begin{equation}
\label{eq:tau-C}
\tau \;\le\;
K \cdot \frac{8\log(\pi^2 A^M M \tau^2/(3\delta))}{\bigl(\Delta_{\sigma(K+1)}^{\langle K\rangle}\bigr)^2}
\;+\;
\sum_{\sigma^{-1}(\bm{a})>K}
\frac{8\log(\pi^2 A^M M \tau^2/(3\delta))}{\bigl(\Delta_{\bm{a}}^{\langle K\rangle}\bigr)^2}.
\end{equation}
\end{theorem}

\paragraph{Discussion.}
Problem C reveals a qualitative gap relative to Problems A and B along two axes: adaptivity and consensus.

\textbf{Adaptivity cost (fixed budget).}
The exponent in Theorem~\ref{thm:algoC-budget} is $-T\Delta_{\bm{a}}^2/(8A^M)$, whereas Problems A and B achieve $-T/(32H_1)$.
These look similar when all arms have the same gap ($\Delta_{\bm{a}} = \Delta$ for all $\bm{a}$, giving $H_1 = A^M/\Delta^2$), but in general $H_1 \le A^M/\Delta_{\min}^2$ with equality only in that degenerate case.
In any \emph{non-uniform} instance---where some arms are far from the boundary and are eliminated quickly by UCB-Intervals---we have $H_1 \ll A^M/\Delta_{\min}^2$, so the UCB-Intervals error for A/B decays much faster than the uniform-exploration error for C.
Put differently, Problem C pays a budget penalty equal to the number of easy arms it is forced to over-explore.

\textbf{Consensus cost (fixed confidence).}
The factor-$2$ enlargement in $\epsilon_{\bm{a}}^i(t)$ is not an artifact of our analysis: it is necessary to guarantee that whenever one player's intervals certify a top-$K$ separation, \emph{every other} player's empirical means are consistent with the same separation (Appendix~\ref{apx:proof-thmCconf}). This quadruples the per-arm sample count, producing the factor-$8$ in Theorem~\ref{thm:algoC-confidence}---a $4\times$ overhead over Problems A/B that we show in Section~\ref{sec:optimality} is the dominant, and near-unavoidable, cost of consensus without communication.

\begin{remark}[Decentralized stopping and safe exit]
\label{rem:C-stopping}
Because rewards are private, the empirical means $\widehat\mu^i_{\bm a}$ differ across players, so players certify at \emph{different} rounds and no data-dependent stopping time is identical team-wide from private information alone. Two facts still make Algorithm~\ref{alg:unif-C-confidence} a valid decentralized procedure. (i) \emph{Correctness.} The round-robin schedule is deterministic and pre-agreed, so the count $n^i_{\bm a}(t)=\lfloor t/A^M\rfloor$ and radius $\epsilon^i_{\bm a}(t)$ are the same known function of $t$ for every player; only realizations differ. Under the good event the enlarged (factor-$2$) radius forces any certified set to be the \emph{true} top-$K$ (Appendix~\ref{apx:proof-thmCconf}), so committed players agree without certifying simultaneously. (ii) \emph{No stream corruption.} A committed player keeps executing its scheduled component (line~4) rather than deviating, so every teammate's stream is exactly as analyzed. Hence by $\tau=\max_i\tau_i$, with $\tau_i$ player $i$'s private certification time, all players hold the correct set, and~\eqref{eq:tau-C} controls $\tau$.

What a player cannot do communication-free is \emph{detect} that the slowest teammate has finished. Two safe exits resolve this: \textbf{(a)} a committed player continues the costless deterministic schedule indefinitely---never corrupting a teammate, correct from round $\tau$ on, but never physically halting; or \textbf{(b)} one end-of-run synchronization bit per player lets all leave the first round every $\texttt{ready}_i$ flag is up, i.e.\ exactly round $\tau$. Option (b) recovers simultaneous termination with a single bit, quantifying the coordination Problem~C removes (cf.\ Remark~\ref{rem:price}).
\end{remark}

\begin{remark}[Can adaptivity be recovered in Problem C?]
\label{rem:C-adaptive}
The uniform rule over-explores because dropping an arm from the schedule changes the joint action and silently corrupts teammates---the corruption Remark~\ref{rem:C-stopping} avoids by never deviating. A player \emph{may} stop \emph{updating} an arm locally once its doubled interval certifies that arm, while still executing the round-robin; this never corrupts a stream and can only speed private certification. It does not improve the worst case, though: when every arm is a boundary arm ($H_1=A^M/\Delta_{\min}^2$) none is safe to freeze and the schedule still visits all $A^M$ arms, recovering Theorem~\ref{thm:algoC-budget}. Whether a non-uniform but teammate-safe schedule can provably beat uniform exploration without communication is open---the deviation channel of Problem~B is exactly the ingredient Problem~C lacks.
\end{remark}

\section{Unified Analysis and Optimality}
\label{sec:optimality}

The three regimes were analyzed separately above, but their fixed-confidence guarantees share a common structure. We now (i) consolidate them into a single meta-theorem parameterized by two scalars, (ii) prove matching information-theoretic lower bounds, and (iii) quantify the resulting optimality gap and the inherent price of decentralization.

\subsection{A unified fixed-confidence guarantee}
\label{sec:unified}

Each regime is captured by two parameters: a \emph{multiplicity} $c \in \{1, M\}$, equal to the number of independent confidence events that must hold simultaneously ($c = 1$ when rewards are shared, $c = M$ when each player maintains its own stream), and a \emph{consensus factor} $\rho \in \{1, 2\}$, the radius inflation needed for unanimous output ($\rho = 1$ when an observable signal synchronizes eliminations, $\rho = 2$ when consensus must be enforced implicitly). Problems A, B, C correspond to $(c,\rho) = (1,1),\,(M,1),\,(M,2)$.

Concretely, the two factors track two independent questions. \emph{How many streams must concentrate?} With shared rewards (A) the whole team rides one reward stream, so a single good event suffices and $c=1$; with private rewards (B, C) each of the $M$ players estimates from its own stream and all $M$ good events must hold at once, so $c=M$ and a union bound over streams enters the log. \emph{How is agreement reached?} When some observable signal synchronizes the team---the shared statistics of A or the deviation broadcast of B---one player's certification is immediately shared, so no radius inflation is needed and $\rho=1$; when neither is available (C), a player must be sure its certified set matches every teammate's, which the doubled radius ($\rho=2$) guarantees. A useful mnemonic on a two-player instance ($M=2$): moving A$\to$B keeps $\rho=1$ but doubles $c$, adding only a $\log 2$ inside the logarithm; moving B$\to$C keeps $c=2$ but doubles $\rho$, multiplying the whole stopping time by $\rho^2=4$. The factor that hurts is the one you cannot hide inside a logarithm.

\begin{theorem}[Unified fixed-confidence guarantee]
\label{thm:unified}
Consider the interval-based fixed-confidence algorithm (Appendix~\ref{apx:fc-algos}) run with per-player radius
\begin{equation}
\label{eq:unified-radius}
\epsilon_{\bm{a}}^{i}(t) \;=\; \rho\sqrt{\frac{2\log\!\bigl(c\,\pi^2 A^M t^2/(3\delta)\bigr)}{n_{\bm{a}}^{i}(t)}}.
\end{equation}
Then the algorithm is $\delta$-correct and, with probability at least $1-\delta$, terminates with all players in agreement after a number of rounds
\begin{equation}
\label{eq:unified-tau}
\tau \;\le\; 2\rho^2\, H_1 \,\log\!\bigl(c\,\pi^2 A^M \tau^2/(3\delta)\bigr)
\;=\; O\!\Bigl(\rho^2 H_1 \log\bigl(c A^M/\delta\bigr)\Bigr).
\end{equation}
Substituting $(c,\rho) = (1,1), (M,1), (M,2)$ recovers Theorems~\ref{thm:algoA-confidence}, \ref{thm:algoB-confidence}, and~\ref{thm:algoC-confidence} respectively.
\end{theorem}
\begin{proof}
See Appendix~\ref{apx:proof-unified}; the three regime-specific proofs are corollaries obtained by the stated substitutions.
\end{proof}

This consolidation isolates the two distinct mechanisms by which information asymmetry inflates sample complexity. The multiplicity $c$ enters only \emph{inside the logarithm}, contributing an additive $O(\rho^2 H_1 \log c/\Delta_{\min}^2)$ term---negligible relative to the leading $O(\rho^2 H_1 \log(A^M)/\Delta_{\min}^2)$. The consensus factor $\rho$ enters \emph{quadratically} outside the logarithm, so it is the dominant cost: moving from $\rho = 1$ (Problems A, B) to $\rho = 2$ (Problem C) multiplies the stopping time by $4$.

\begin{remark}[High-probability versus expected stopping time]
\label{rem:expectation}
The radius~\eqref{eq:unified-radius} is a time-uniform confidence sequence: the good event holds at all rounds simultaneously with probability $\ge 1-\delta$, so $\tau$ is finite almost surely. This converts to an expectation bound of the same order: run alongside a fallback that forces a stop at the fixed-budget horizon $T_\delta=O(\rho^2 H_1\log(cA^M/\delta))$ (Theorems~\ref{thm:algoA-budget}--\ref{thm:algoC-budget}, which already attain error $\le\delta$ there). Then $\mathbb{E}[\tau]\le T_\delta$, matching Theorem~\ref{thm:lower-fc} up to the exploration log-factor, so the guarantee is not merely high-probability.
\end{remark}

\subsection{Information-theoretic lower bounds}
\label{sec:lower}

We now show these guarantees are near-optimal. Our lower bounds use the change-of-measure (transportation) technique of \cite{kaufmann2016complexity}. We take Gaussian rewards of unit variance; for two instances $\nu, \nu'$ the per-arm relative entropy is $\mathrm{KL}(\nu_{\bm{a}}, \nu'_{\bm{a}}) = (\mu_{\bm{a}} - \mu'_{\bm{a}})^2/2$.

\begin{lemma}[Transportation inequality]
\label{lem:transport}
Let $\mathcal{B}$ be any $\delta$-correct algorithm, i.e.\ $\mathbb{P}_\nu(\widehat{S} \neq S^\star(\nu)) \le \delta$ for every instance $\nu$ in the class. For any alternative $\nu'$ with $S^\star(\nu') \neq S^\star(\nu)$,
\[
\sum_{\bm{a}} \mathbb{E}_\nu[N_{\bm{a}}(\tau)]\,\mathrm{KL}(\nu_{\bm{a}}, \nu'_{\bm{a}})
\;\ge\; \mathrm{kl}(\delta, 1-\delta) \;\ge\; \log\!\tfrac{1}{2.4\delta},
\]
where $N_{\bm{a}}(\tau)$ is the number of rounds in which the prescribed joint action is $\bm{a}$, and $\mathrm{kl}(x,y) = x\log\tfrac{x}{y} + (1-x)\log\tfrac{1-x}{1-y}$.
\end{lemma}
\begin{proof}
See Appendix~\ref{apx:proof-transport}.
\end{proof}

\begin{theorem}[Fixed-confidence lower bound]
\label{thm:lower-fc}
Fix any instance $\nu$ with distinct means and unit-variance Gaussian rewards. Every $\delta$-correct algorithm---for any of Problems A, B, C---satisfies
\begin{equation}
\label{eq:lower-fc}
\mathbb{E}_\nu[\tau] \;\ge\; 2\,H_1 \log\!\tfrac{1}{2.4\delta},
\qquad
H_1 = \sum_{\bm{a}} \bigl(\Delta_{\bm{a}}^{\langle K\rangle}\bigr)^{-2}.
\end{equation}
\end{theorem}
\begin{proof}
For each arm $\bm{a}$ build an alternative $\nu^{(\bm{a})}$ that perturbs only $\mu_{\bm{a}}$ by exactly its gap $\Delta_{\bm{a}}^{\langle K\rangle}$: a sub-optimal arm is raised just above $\mu_{\sigma(K)}$, or a top-$K$ arm is lowered just below $\mu_{\sigma(K+1)}$. Either change alters the top-$K$ set, so Lemma~\ref{lem:transport} applies. Since only coordinate $\bm{a}$ differs, $\mathbb{E}_\nu[N_{\bm{a}}(\tau)]\,(\Delta_{\bm{a}}^{\langle K\rangle})^2/2 \ge \log\frac{1}{2.4\delta}$, i.e.\ $\mathbb{E}_\nu[N_{\bm{a}}(\tau)] \ge 2\log\frac{1}{2.4\delta}/(\Delta_{\bm{a}}^{\langle K\rangle})^2$. Summing over all $\bm{a}$ and using $\mathbb{E}_\nu[\tau] = \sum_{\bm{a}} \mathbb{E}_\nu[N_{\bm{a}}(\tau)]$ yields~\eqref{eq:lower-fc}. For Problems B and C the identical construction applies to any single player's reward stream: each player must output the correct set with probability $\ge 1-\delta$ using only its own observations, so its sample count obeys the same bound. See Appendix~\ref{apx:proof-lower}.
\end{proof}

\begin{remark}[Fixed-budget lower bound]
\label{rem:lower-fb}
A matching fixed-budget statement follows from \cite{carpentier2016tight}: any algorithm incurs error probability at least $\exp(-O(T/H_1))$ on some instance in the class, so the rate $\exp(-T/(32H_1))$ of Theorem~\ref{thm:algoA-budget} is optimal up to the constant in the exponent (the complexity measures $H_1$ and the sorted-gap measure $H_2$ agree up to a $\log(A^M)$ factor).
\end{remark}

\subsection{Near-optimality and the price of decentralization}
\label{sec:price}

\begin{corollary}[Near-optimality]
\label{cor:optimal}
Combining Theorem~\ref{thm:unified} with the lower bound~\eqref{eq:lower-fc}, the ratio of achieved to optimal stopping time is
\[
\frac{\tau_{\text{alg}}}{\mathbb{E}_\nu[\tau]^{\text{lower}}}
\;\le\;
\rho^2\cdot\frac{\log(c\,\pi^2 A^M \tau^2/(3\delta))}{\log(1/(2.4\delta))}
\;\xrightarrow{\;\delta\to 0\;}\;
\rho^2 \cdot \Bigl(1 + \tfrac{M\log A + 2\log\tau + O(1)}{\log(1/\delta)}\Bigr).
\]
Thus Problem~A ($\rho=1$) is asymptotically optimal up to the universal exploration log-factor shared by all UCB-type fixed-confidence algorithms; Problem~B matches the same factor with a negligible additive $\log M$; and Problem~C is optimal up to the multiplicative constant $\rho^2 = 4$ forced by the consensus radius.
\end{corollary}

\begin{remark}[Price of decentralization and an open gap]
\label{rem:price}
In Problems B and C each round yields $M$ independent samples, so an oracle permitted to \emph{pool} observations across players would meet the transportation bound in $\Theta(H_1\log(1/\delta)/M)$ rounds---an $M$-fold speedup over~\eqref{eq:lower-fc}. Our algorithms do not pool, and~\eqref{eq:lower-fc} confirms no \emph{communication-free} algorithm can. Whether Problem~B's deviation channel---a genuine low-bandwidth primitive---can be repurposed (e.g.\ partitioning arms across players and broadcasting partial rankings through deviation patterns) to approach the pooled rate is an open question our framework makes precise. An $(\varepsilon,\delta)$-PAC relaxation appears in Appendix~\ref{apx:pac}.
\end{remark}

\section{Experiments}
\label{sec:experiments}

We empirically evaluate all three algorithms on a synthetic instance with $M=2$ players, $A=3$ individual actions each, and $K=3$ top arms to identify, giving $N=A^M=9$ joint arms. Mean rewards are $\bm\mu = [0.80,\,0.75,\,0.72,\,0.50,\,0.45,\,0.40,\,0.35,\,0.30,\,0.25]$ (arms indexed in decreasing order), so the top-$3$ set is $\{0,1,2\}$. The critical gap between rank-$3$ and rank-$4$ is $\Delta = 0.22$, and the complexity constant is $H_1 \approx 109$. Rewards are Gaussian with unit variance.

\paragraph{Fixed-budget experiment.}
Figure~\ref{fig:main}(a) shows empirical error probability versus budget $T$ on a CTR-style instance with $N=50$ joint arms and a top-$K$ boundary gap of $0.06$ (80 trials per point; $K=3$, Bernoulli click rewards). Problem~B has the highest error at small budgets ($T\lesssim 6{,}000$): with $M=2$ private reward streams per round, the deviation signal fires whenever \emph{any} player prematurely certifies domination, inflating false eliminations. Problem~C (uniform exploration) attains the lowest error at small-to-moderate budgets, since each of the $M$ streams contributes to every arm's estimate. Problem~A overtakes both around $T\approx 10{,}000$ and reaches zero error first, as adaptive elimination concentrates the remaining budget on the hard boundary arms---ultimately winning the information-versus-adaptivity trade-off.

\paragraph{Fixed-confidence experiment.}
Figure~\ref{fig:main}(b) plots stopping time $\tau$ versus confidence level $\delta$ on a log--log scale, on a smaller CTR instance ($N=24$, boundary gap $0.15$; 12 trials for Problems~A/B, 4 for Problem~C). Problems~A and~B stop at essentially the same time ($\tau_A\approx\tau_B\approx 82\text{--}88{,}000$): the deviation channel of B matches A's coordination at $M=2$. Problem~C requires $\tau_C\approx 1.3\times10^{6}$ rounds---roughly $15\times$ longer than A---because its uniform rule keeps sampling all $N$ arms and its enlarged confidence radius demands far more samples to certify the top-$K$ separation (see Theorems~\ref{thm:algoA-confidence}--\ref{thm:algoC-confidence}). All stopping times grow only logarithmically in $1/\delta$, so the three confidence levels lie close together, matching the $O(\log(1/\delta)/\Delta^2)$ scaling from theory.

\paragraph{Scalability with $M$.}
Figure~\ref{fig:scale}(c) shows error versus the joint-arm count $N=A^M$ with $A=3$ and a \emph{fixed total budget} $T=24{,}000$, so the per-arm budget shrinks as $N$ grows from $9$ ($M=2$) to $243$ ($M=5$). All three methods are essentially exact for $M\le 3$ and degrade once $N$ outstrips the budget, confirming the exponential $A^M$ dependence predicted by theory. Problem~A degrades most gracefully (error $\le 0.07$ even at $M=5$), Problem~C is intermediate ($\to 0.65$), and Problem~B suffers most ($\to 0.98$), reflecting the $M$-fold union bound over its independent elimination signals.

\paragraph{Scope of the evaluation and the $A^M\gg T$ regime.}
Our instances use small joint spaces ($A^M\le 50$ in the fixed-budget and fixed-confidence panels, up to $A^M=243$ in the scalability panel) because the fixed-confidence runs, especially Problem~C, are already computationally heavy at these sizes. This is a genuine limitation for validating behavior when $A^M$ grows large, and in particular when $A^M\gg T$: there the budget for each arm, $\lfloor T/A^M\rfloor$, falls below one, no arm is sampled enough to separate, and \emph{every} unstructured method must fail---this is not a defect of a particular algorithm but the information-theoretic wall of Theorem~\ref{thm:lower-fc} (the scalability panel already shows all three methods collapsing as $N$ passes the budget). The intended escape from this wall is structural rather than algorithmic: when rewards decompose additively, Section~\ref{sec:additive} reduces the effective arm count from $A^M$ to $O(MA)$, so the same budget suffices for far larger $M$. We regard large-$A^M$ evaluation under such structure, and stress-testing the additive reconstruction, as the natural next empirical step.

\begin{figure}[t]
  \centering
  \begin{minipage}[t]{0.63\textwidth}
    \includegraphics[width=\linewidth]{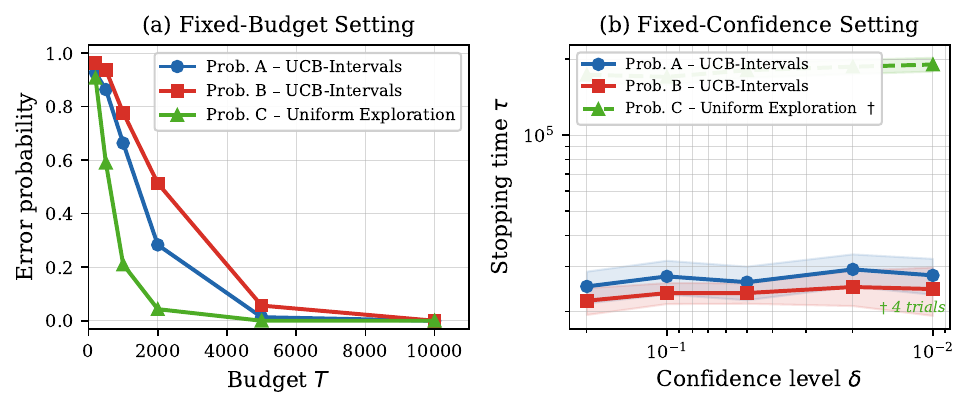}
    \caption{%
      \textbf{(a)} Error probability vs.\ budget $T$ on a CTR instance ($N=50$ joint arms, $K=3$, boundary gap $0.06$).
      \textbf{(b)} Stopping time vs.\ confidence level~$\delta$ (log--log) on a CTR instance ($N=24$); the curve for
      Problem~C uses $\approx$4 trials per point due to computational cost.}
    \label{fig:main}
  \end{minipage}
  \hfill
  \begin{minipage}[t]{0.34\textwidth}
    \includegraphics[width=\linewidth]{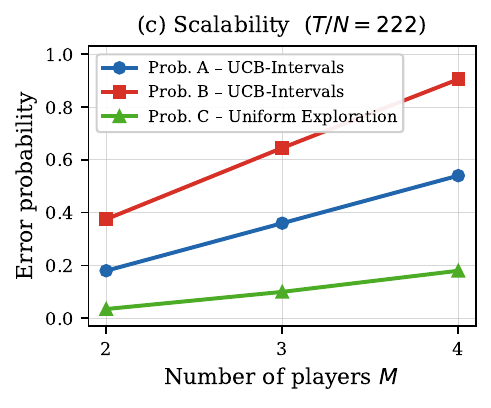}
    \caption{%
      Error vs.\ joint-arm count $N=A^M$ ($A=3$) at fixed total budget $T=24{,}000$.
      Problem~A degrades most gracefully; Problem~B suffers most from the $M$-fold
      false-alarm penalty.}
    \label{fig:scale}
  \end{minipage}
\end{figure}

\section{Structured Rewards: Additive Decomposition}
\label{sec:additive}

The bounds above scale with the joint cardinality $A^M$. This dependence is unavoidable in the worst case: the lower bound of Theorem~\ref{thm:lower-fc} gives $\mathbb{E}_\nu[\tau]\ge 2H_1\log\tfrac{1}{2.4\delta}$ with $H_1=\sum_{\bm a}(\Delta^{\langle K\rangle}_{\bm a})^{-2}$ a sum of $A^M$ terms (so $H_1\ge A^M/\Delta_{\max}^2$ for the largest gap $\Delta_{\max}$), matching the reduction to single-player top-$K$ over $A^M$ arms in Section~\ref{sec:prelim}. Many instances are nonetheless \emph{separable}---the joint reward decomposes additively across players---and we show this collapses the exponential arm count to a polynomial one.

\begin{assumption}[Additive rewards]\label{ass:additive}
There exist $\theta_i:\mathcal{A}\to\mathbb{R}$, $i\in[M]$, with $\mu_{\bm a}=\sum_{i=1}^M\theta_i(a_i)$ for all $\bm a$.
\end{assumption}

Fix a reference profile $\bm r$ and write $\bm r_{-i}(a)$ for $\bm r$ with coordinate $i$ reset to $a$. Each contrast satisfies $\theta_i(a)-\theta_i(r_i)=\mu_{\bm r_{-i}(a)}-\mu_{\bm r}$, so the $MA$ \emph{reference arms} $\{\bm r_{-i}(a)\}$ determine every joint mean via $\mu_{\bm a}=\sum_i\mu_{\bm r_{-i}(a_i)}-(M-1)\mu_{\bm r}$, up to a common constant irrelevant to ranking. \texttt{UCB-Intervals} restricted to these $MA$ arms (Algorithm~\ref{alg:additive}, Appendix~\ref{apx:additive}) reconstructs all $A^M$ confidence intervals without ever sampling the remaining arms.

\begin{theorem}[Additive structure]\label{thm:additive}
Under Assumption~\ref{ass:additive}, the reference-design variant returns the correct top-$K$ with probability $\ge 1-\delta$ within
\[
\tau \;=\; O\!\left( \frac{M^3 A\,\log(MA/\delta)}{\Delta_{\min}^2} \right)\ \text{rounds},
\qquad \Delta_{\min}:=\!\!\min_{\bm a:\,\Delta^{\langle K\rangle}_{\bm a}>0}\!\Delta^{\langle K\rangle}_{\bm a}.
\]
Hence additive structure reduces the complexity from exponential ($\Theta(A^M)$ effective arms) to polynomial in $M$ and $A$.
\end{theorem}

The proof (Appendix~\ref{apx:additive}) reconstructs each joint estimate as a fixed linear combination of $M{+}1$ reference estimators; the reconstruction inflates the noise variance by a factor $O(M^2)$, and a time-uniform union bound over the $MA$ reference arms (rather than over $A^M$ arms) yields the stated rate. The same reduction applies verbatim to the fixed-budget error bound of Theorem~\ref{thm:algoA-budget}, with $A^M$ replaced by $MA$ in its prefactor.

\subsection{The price of players: a tight characterization}
\label{sec:additive-tight}
The reference design above is communication-free but pays an $O(M^3)$ overhead. For identifying the single best joint arm ($K=1$) this overhead is an artifact of estimating joint means one reference arm at a time: letting every coordinate vary at once is optimal, and its complexity does not grow with the number of players.

\begin{theorem}[Tight best-arm complexity under additivity]\label{thm:additive-tight}
Adopt Assumption~\ref{ass:additive}, let $K=1$, and let $\Delta_{\min}$ be the smallest per-player gap between a player's best and second-best action value, assumed positive. The stratified design of Algorithm~\ref{alg:stratified} (Appendix~\ref{apx:additive-tight}) returns the best joint arm with probability $\ge 1-\delta$ within $\tau=O\!\big(A\log(MA/\delta)/\Delta_{\min}^2\big)$ rounds. Conversely, on a family of additive instances with gap $\Delta_{\min}$ every $\delta$-PAC algorithm obeys $\mathbb{E}[\tau]\ge \tfrac{A-1}{2\Delta_{\min}^2}\log\frac{1}{2.4\delta}$. The best-arm complexity under additivity is thus $\widetilde\Theta(A/\Delta_{\min}^2)$, \emph{independent of $M$}, and the reference design of Theorem~\ref{thm:additive} is loose by $\Theta(M^3)$ in this regime.
\end{theorem}

\noindent Because a single joint pull measures every player's contribution at once, evidence about the $M$ players accrues in parallel rather than sequentially; the matching lower bound applies the change-of-measure inequality behind Theorem~\ref{thm:lower-fc} to the $A-1$ single-coordinate deviations of one player (Appendix~\ref{apx:additive-tight}).

\section{Conclusions}
We studied top-$K$ joint-arm identification in multi-agent multi-armed bandits under three information-asymmetry regimes. Our main contribution is showing that coordination and information aggregation can be recovered---to different extents---without communication, by exploiting whatever common information remains in each regime.

For Problem A, shared rewards plus a fixed ordering preserve implicit coordination, enabling synchronous eliminations despite unobservable actions. For Problem B, a single observable deviation acts as a public elimination signal, synchronizing heterogeneous candidate sets without explicit messaging. For Problem C, where neither coordination nor sharing is available, uniform exploration with appropriately enlarged confidence radii still guarantees correct identification and unanimous agreement at stopping. A unified meta-theorem (Theorem~\ref{thm:unified}) shows the three regimes differ only through a multiplicity $c$ and a consensus factor $\rho$, and matching change-of-measure lower bounds (Theorem~\ref{thm:lower-fc}) prove Problem~A optimal up to the universal exploration log-factor while pinning the cost of full asymmetry to the multiplicative constant $\rho^2 = 4$. All bounds scale with $A^M$ (and potentially $K^M$)---unavoidable by reduction to the single-player top-$K$ problem on $A^M$ arms.

\paragraph{Future work.}
The most concrete open problem is the one made precise by Remark~\ref{rem:price}: whether the deviation channel of Problem~B can coordinate exploration to approach the pooled $\Theta(H_1\log(1/\delta)/M)$ rate, closing the $M$-fold gap between communication-free and oracle algorithms. Further directions include adaptive joint-action orderings that improve practical efficiency while retaining decentralization; extensions to nonstationary or adversarial rewards; and structured instances (e.g.\ product or low-rank reward tensors) where the joint complexity $H_1$ is far smaller than the worst-case $A^M/\Delta_{\min}^2$.

\begin{ack}
[Acknowledgments to be added in the camera-ready version.]
\end{ack}

\newpage
\bibliography{main}
\bibliographystyle{abbrv}

\appendix

\section{Related Works}\label{sec:related}

\paragraph{Single-agent pure exploration.}
Pure exploration in stochastic multi-armed bandits has been extensively studied under both fixed-confidence and fixed-budget objectives. In the fixed-confidence setting, the goal is to identify a target set of arms with probability at least $1-\delta$ while minimizing the expected number of samples; in the fixed-budget setting, the learner is given a sampling budget and aims to minimize the probability of misidentification. Best-arm identification is the canonical problem in this literature \cite{audibert2010best, gabillon2012best, kaufmann2016complexity}, and has led to a rich set of algorithmic techniques based on confidence intervals, elimination rules, and instance-dependent allocation strategies. Classical elimination-style procedures such as Sequential Halving and Successive Rejects give strong guarantees in fixed-budget regimes and remain important algorithmic primitives for pure exploration \cite{karnin2013almost, jamieson2014lil}. Subsequent work refined the sample complexity of best-arm identification through sharper confidence bounds, lower-bound matching algorithms, and large-deviation analyses \cite{chen2015optimal, locatelli2016an, hassidim2020optimal, wang2023best}. Information-theoretic characterizations further clarified how the difficulty of an instance depends on the gaps among arms and the optimal sampling proportions \cite{kaufmann2013information, kaufmann2016complexity}.

Beyond identifying a single best arm, many works study richer identification objectives. Multiple-arm and top-$K$ identification ask the learner to recover the best subset of arms rather than a single maximizer \cite{bubeck2013multiple, chen2017nearly, jun2016top}. PAC variants relax exact identification and allow the learner to output an approximately optimal arm or subset, often yielding improved sample complexity when small errors are acceptable \cite{zhou2014optimal, kalyanakrishnan2012pac, agarwal2017learning, abbasi2018best}. Other extensions consider thresholding, quantile objectives, heavy-tailed rewards, delayed feedback, and limited adaptivity \cite{locatelli2016an, zhang2021quantile, yu2018pure, grover2018best, agarwal2017learning}. A parallel line studies structured pure exploration, where feasible answers are constrained by combinatorial structure such as matroids, matchings, paths, or other families of subsets \cite{chen2014combinatorial, chen2016pure, rejwan2020top}. Related linear and combinatorial bandit formulations exploit low-dimensional structure to reduce exploration cost \cite{soare2014best, tao2018best, wang2021fast}. These works provide the foundation for sample-efficient identification, but they primarily consider a single learner who directly observes all samples and controls the full sampling process. In contrast, our setting involves multiple agents whose observations and actions may be distributed across a joint action space.

\paragraph{Distributed and cooperative pure exploration.}
A growing literature studies pure exploration when samples are collected by multiple learners or processors. Early distributed exploration models analyze how several agents can accelerate best-arm identification by sharing observations or recommendations \cite{hillel2013distributed, shahrampour2017sequential}. More recent collaborative best-arm identification algorithms quantify the tradeoff between sample complexity, communication, and parallel speedup \cite{jha2022collaborative, karpov2023communication}. Related work considers multi-agent best-arm identification with private communication channels, limited communication rounds, or noisy cooperative observations \cite{agarwal2022multi, rio2023multi, vannella2023best, chang2023optimal}. These models are closest in spirit to our work because they also study identification rather than regret minimization in multi-agent systems.

However, most cooperative pure-exploration models assume that the agents share a common arm set, that each agent can sample arms independently, or that information can eventually be pooled through an explicit communication protocol. Under these assumptions, the central challenge is usually to reduce communication while preserving the sample complexity of a centralized learner. Our problem differs in that agents interact through joint actions and may not have symmetric access to the information needed to evaluate every joint action. As a result, the difficulty is not only statistical allocation across arms, but also coordination under asymmetric information. This distinction makes existing distributed best-arm identification algorithms insufficient for our setting, even when their sample complexity guarantees are near-optimal in classical shared-arm models.

\paragraph{Multi-player regret-minimization bandits.}
Multi-player bandit models have also been widely studied in the regret-minimization setting. In these problems, several players repeatedly choose arms and may suffer collisions when multiple players select the same arm. A central objective is to learn high-reward assignments while avoiding collisions and minimizing cumulative regret. Algorithms such as Musical Chairs, SIC-MMAB, and DPE1 use collisions as implicit communication signals and achieve coordination without requiring explicit message passing \cite{rosenski2016multi, boursier2019sic, wang2020optimal}. Other works study decentralized learning with independent rewards, heterogeneous preferences, finite shareable resources, or no collision sensing \cite{kalathil2012multi, kalathil2014decentralized, nayyar2016regret, avner2019multi, shi2020decentralized, shi2021heterogeneous, wang2022multi, mehrabian2020practical, lugosi2022multiplayer}. These papers demonstrate that decentralized coordination is possible under carefully designed feedback models, and they provide important algorithmic ideas for communication through actions.

The objectives and feedback assumptions in regret-minimization multi-player bandits are nevertheless different from ours. Regret algorithms optimize online performance over time, whereas pure exploration aims to identify a correct answer after a learning phase. Moreover, collision-based approaches often rely on synchronized rounds, shared rankings, or the ability to encode information through deliberately chosen collisions. These mechanisms are not directly applicable when the target is fixed-confidence or fixed-budget identification over joint action spaces, especially when agents have asymmetric observations about the rewards or feasibility of joint actions. Our work is therefore complementary to the regret literature: rather than minimizing cumulative loss during repeated play, we study how agents can gather enough information to identify the desired joint action or set of joint actions.

\paragraph{Decentralized coordination and communication constraints.}
A broader body of work studies distributed bandit learning and decentralized decision-making under communication constraints. Gossip-based and distributed stochastic bandit algorithms analyze how agents can aggregate information over networks while limiting communication cost \cite{szorenyi2013gossip, landgren2021distributed}. Multi-agent bandit models with limited communication study how much information must be exchanged to approach centralized performance \cite{agarwal2022multi, karpov2023communication}. In related multiplayer settings, communication may be implicit through collisions, arm choices, or other structured feedback rather than explicit messages \cite{boursier2019sic, rosenski2016multi, bistritz2018distributed}. These works highlight a recurring theme: statistical efficiency is tightly coupled with the information structure available to the agents.

Our setting shares this emphasis on information structure, but differs in the source of decentralization. Rather than assuming that each player faces the same arm set and must coordinate to avoid collisions, we consider identification over joint actions where each agent's local view may reveal only part of the relevant information. This creates an asymmetric-information pure-exploration problem: the agents must reason about a global objective even though no individual agent necessarily observes all components of the joint reward landscape. Consequently, algorithms designed for shared-arm distributed exploration, collision-based communication, or decentralized regret minimization do not directly resolve the main challenge we address.

\paragraph{Positioning of our contribution.}
Taken together, prior work provides powerful tools for single-agent identification, combinatorial pure exploration, collaborative best-arm identification, and decentralized multi-player learning. The closest related directions are collaborative pure exploration \cite{jha2022collaborative, rio2023multi, vannella2023best, karpov2023communication} and decentralized multi-player bandits with implicit coordination \cite{boursier2019sic, rosenski2016multi, wang2020optimal, lugosi2022multiplayer}. However, these lines typically assume either centralized access to the relevant observations after communication, a common arm set, or structured collision feedback that can be used for coordination. Our work instead focuses on pure exploration with multiple agents, asymmetric information, and joint action spaces. This combination requires new algorithmic ideas beyond classical elimination, distributed averaging, or collision-based coordination, and it leads to sample-complexity questions that are not captured by existing single-agent or shared-arm multi-agent formulations.

\section{Additive Reward Structure: Algorithm and Proof}
\label{apx:additive}

\begin{algorithm}[H]
\caption{Reference-design \texttt{UCB-Intervals} under additive rewards}
\label{alg:additive}
\begin{algorithmic}[1]
\Require budget/confidence target, target $K$, reference profile $\bm r$
\State $\mathcal{R}\gets\{\bm r\}\cup\{\bm r_{-i}(a): i\in[M],\,a\in\mathcal{A}\}$ \Comment{$|\mathcal R|\le MA+1$}
\State Round-robin over $\mathcal{R}$, updating $\widehat\mu_{\bm s},n_{\bm s}$ for each $\bm s\in\mathcal R$
\State For every joint arm $\bm a$, set $\widehat\mu_{\bm a}\gets\sum_{i=1}^M\widehat\mu_{\bm r_{-i}(a_i)}-(M-1)\widehat\mu_{\bm r}$
\State Maintain $I_{\bm a}=[\widehat\mu_{\bm a}-\epsilon_{\bm a},\widehat\mu_{\bm a}+\epsilon_{\bm a}]$ with $\epsilon_{\bm a}$ from the reconstruction variance below; eliminate / certify the top-$K$ exactly as in Algorithms~\ref{alg:ucb-A-budget}/\ref{alg:ucb-A-confidence}
\end{algorithmic}
\end{algorithm}

\begin{proof}[Proof of Theorem~\ref{thm:additive}]
Write $\bar\theta_i(a):=\theta_i(a)-\theta_i(r_i)$. Since $\mu_{\bm r_{-i}(a)}=\theta_i(a)+\sum_{j\ne i}\theta_j(r_j)$ and $\mu_{\bm r}=\sum_j\theta_j(r_j)$, we have $\mu_{\bm r_{-i}(a)}-\mu_{\bm r}=\bar\theta_i(a)$, and for any $\bm a$,
\[
\mu_{\bm a}-\mu_{\bm r}=\sum_{i=1}^M\bar\theta_i(a_i)=\sum_{i=1}^M\big(\mu_{\bm r_{-i}(a_i)}-\mu_{\bm r}\big),
\]
which rearranges to the reconstruction $\mu_{\bm a}=\sum_i\mu_{\bm r_{-i}(a_i)}-(M-1)\mu_{\bm r}$. Replacing each term by its empirical mean defines $\widehat\mu_{\bm a}$. The reference estimators $\{\widehat\mu_{\bm s}\}_{\bm s\in\mathcal R}$ use disjoint samples, hence are independent and each $1/n$-subgaussian after $n$ pulls. The reconstruction places weight $1$ on the (at most) $M$ distinct arms $\bm r_{-i}(a_i)$ and weight $-(M-1)$ on $\bm r$, so
\[
\mathrm{Var}(\widehat\mu_{\bm a})\le \frac{1}{n}\Big(M\cdot 1^2+(M-1)^2\Big)=\frac{1}{n}\,O(M^2).
\]
Thus $\widehat\mu_{\bm a}$ is $O(M^2)/n$-subgaussian, and with $\epsilon_{\bm a}(n):=\sqrt{2 \kappa_0 M^2\log(\pi^2 MA\, t^2/(3\delta))/n}$ for an absolute constant $\kappa_0$ the event $\bigcap_{\bm a,t}\{|\widehat\mu_{\bm a}-\mu_{\bm a}|\le\epsilon_{\bm a}\}$ holds with probability $\ge 1-\delta$ by a time-uniform union bound over the $|\mathcal R|\le MA+1$ reference arms (\emph{not} over $A^M$). Certifying the top-$K$ requires $\epsilon_{\bm a}\le \Delta^{\langle K\rangle}_{\bm a}/2$ at the boundary arms; solving for $n$ gives $n=O\!\big(M^2\log(MA/\delta)/\Delta_{\min}^2\big)$ samples per reference arm. Summing the round-robin pulls over $|\mathcal R|=O(MA)$ reference arms gives
$\tau=O(MA)\cdot O\!\big(M^2\log(MA/\delta)/\Delta_{\min}^2\big)=O\!\big(M^3A\log(MA/\delta)/\Delta_{\min}^2\big)$.
Correctness on the good event follows exactly as in the proof of Theorem~\ref{thm:algoA-confidence}.
\end{proof}

\section{Tight Best-Arm Complexity under Additivity}
\label{apx:additive-tight}

\begin{algorithm}[H]
\caption{Stratified design for the best joint arm under additive rewards}
\label{alg:stratified}
\begin{algorithmic}[1]
\Require confidence $\delta$ (or a fixed horizon)
\For{$t=1,2,\dots$}
  \State Each player $i$ independently draws $a_i(t)\sim\mathrm{Unif}([A])$; all players play $\bm a(t)$ and observe $y_t=\mu_{\bm a(t)}+\eta_t$ with $1$-subgaussian noise $\eta_t$.
\EndFor
\For{each $i\in[M]$, $a\in[A]$}
  \State $\widehat m_{i,a}\gets \mathrm{mean}\{\,y_t : a_i(t)=a\,\}$
\EndFor
\State For each $i$, set $\widehat a_i\gets\arg\max_a \widehat m_{i,a}$; \Return $\widehat{\bm a}=(\widehat a_1,\dots,\widehat a_M)$
\end{algorithmic}
\end{algorithm}

\begin{proof}[Proof of the upper bound in Theorem~\ref{thm:additive-tight}]
Fix player $i$ and write $a_i^\star$ for its unique best action. Conditioned on the design, for any $a$,
\[
\mathbb{E}[\widehat m_{i,a}\mid a_i(t)=a]=\theta_i(a)+\textstyle\sum_{j\ne i}\mathbb{E}_{a_j\sim\mathrm{Unif}}[\theta_j(a_j)]=\theta_i(a)+(M-1)\bar\theta,
\]
with $\bar\theta:=\tfrac1A\sum_{a'}\tfrac1M\sum_j\theta_j(a')$ a constant free of $(i,a)$. Thus $\widehat m_{i,a}-\widehat m_{i,a'}$ is unbiased for $\theta_i(a)-\theta_i(a')$: the baseline cancels and the other $M-1$ coordinates do not bias the contrast. After $n$ rounds, $n_{i,a}:=|\{t\le n:a_i(t)=a\}|\sim\mathrm{Binomial}(n,1/A)$, so $n_{i,a}\ge n/(2A)$ for every $i,a$ with probability $\ge 1-MAe^{-n/(8A)}$. On that event each $\widehat m_{i,a}$ averages $\ge n/(2A)$ independent 1-subgaussian samples, and a union bound over the $MA$ pairs gives, with probability $\ge 1-\delta$,
\[
\big|\widehat m_{i,a}-\theta_i(a)-(M-1)\bar\theta\big|\le \sqrt{\tfrac{2A}{n}\log\tfrac{4MA}{\delta}}\qquad\forall i,a.
\]
When this deviation is $<\Delta_{\min}/2$ the empirical per-player ranking is exact, so $\widehat a_i=a_i^\star$ for all $i$ and $\widehat{\bm a}=\bm a^\star$. Solving for $n$ gives $n=O(A\log(MA/\delta)/\Delta_{\min}^2)$, which also dominates the $O(A\log(MA/\delta))$ rounds needed for the balance event.
\end{proof}

\begin{proof}[Proof of the lower bound in Theorem~\ref{thm:additive-tight}]
Take the additive instance $\theta_i(1)=\Delta_{\min}$, $\theta_i(a)=0$ ($a\ge2$), $i\in[M]$, with unit-variance Gaussian noise; its unique best joint arm is $\bm a^\star=(1,\dots,1)$. Fix $i$ and, for $a\ge2$, let $\lambda^{(i,a)}$ raise player $i$'s value of $a$ to $2\Delta_{\min}$, making the best joint arm differ from $\bm a^\star$ in coordinate $i$; any $\delta$-PAC algorithm must separate $\mu$ from each $\lambda^{(i,a)}$. The instances differ only on arms with $a_i=a$, where the Gaussian KL is $2\Delta_{\min}^2$, so the transportation inequality (Kaufmann--Capp\'e--Garivier, underlying Theorem~\ref{thm:lower-fc}) yields
\[
2\Delta_{\min}^2\,\mathbb{E}_\mu[T_{i,a}(\tau)]\ge \mathrm{kl}(\delta,1-\delta)\ge\log\tfrac{1}{2.4\delta},\qquad T_{i,a}(\tau):=\!\!\sum_{\bm a:a_i=a}\!\!N_{\bm a}(\tau).
\]
The events $\{a_i(t)=a\}$, $a\ge2$, are disjoint, so $\sum_{a\ge2}T_{i,a}(\tau)\le\tau$ and summing the $A-1$ inequalities gives $\mathbb{E}_\mu[\tau]\ge \tfrac{A-1}{2\Delta_{\min}^2}\log\tfrac{1}{2.4\delta}$.
\end{proof}

\begin{remark}
The bound is $M$-free because the $M$ players' deviations constrain the \emph{same} pulls: one allocation with each coordinate uniform over the suboptimal actions makes all $M(A-1)$ alternatives simultaneously hard. The $O(M^2)$ variance inflation of the reference design is avoidable, but only by varying all coordinates at once.
\end{remark}

\section{Fixed-Confidence Algorithms}
\label{apx:fc-algos}

The fixed-confidence variants of all three algorithms replace the fixed-$T$ loop with a \texttt{while} loop that terminates once the top-$K$ set is statistically certified. The termination condition and confidence radii are chosen to ensure $\ge 1-\delta$ correctness via a time-uniform union bound (see Section~\ref{apx:proofs}).

\begin{algorithm}[H]
\caption{\texttt{UCB-Intervals} for Problem A (Fixed Confidence)}
\label{alg:ucb-A-confidence}
\begin{algorithmic}[1]
\Require $A^M$, target $K$, confidence $\delta$
\State $\mathcal{D} \gets \mathcal{A}$; fix an ordering; $I_{\bm{a}} \gets (-\infty,\infty)$ for all $\bm{a}$
\While{$|\mathcal{D}| > K$}
    \State Select next $\bm{a} \in \mathcal{D}$ per ordering
    \While{$\exists\,\mathcal{S}\!\subseteq\!\mathcal{D},\,|\mathcal{S}|\!=\!K : \forall\bm{b}\!\in\!\mathcal{S},\;\min I_{\bm{b}} > \max I_{\bm{a}}$}
        \State $\mathcal{D} \gets \mathcal{D}\setminus\{\bm{a}\}$; select next $\bm{a} \in \mathcal{D}$
    \EndWhile
    \State All players pull $\bm{a}[i]$; update $I_{\bm{a}}$ with $\epsilon_{\bm{a}}(t) = \sqrt{2\log(\pi^2 A^M t^2/(3\delta))/n_{\bm{a}}(t)}$
\EndWhile
\State \Return $\mathcal{D}$
\end{algorithmic}
\end{algorithm}

\begin{algorithm}[H]
\caption{\texttt{UCB-Intervals} for Problem B (Fixed Confidence)}
\label{alg:ucb-B-confidence}
\begin{algorithmic}[1]
\Require $A^M$, target $K$, confidence $\delta$
\State $\mathcal{D} \gets \mathcal{A}$; fix an ordering; $I_{\bm{a}}^{i} \gets (-\infty,\infty)$ for all $i, \bm{a}$
\While{$|\mathcal{D}| > K$}
    \State Select next $\bm{a} \in \mathcal{D}$ per ordering
    \For{each player $i$}
        \If{$\exists\,\mathcal{S}\!\subseteq\!\mathcal{D},\,|\mathcal{S}|\!=\!K : \forall\bm{b}\!\in\!\mathcal{S},\;\min I_{\bm{b}}^{i} > \max I_{\bm{a}}^{i}$}
            \State Pull action $\neq \bm{a}[i]$ \hfill\Comment{deviation = elimination signal}
        \Else\ Pull $\bm{a}[i]$
        \EndIf
    \EndFor
    \For{each player $i$}
        \State Observe $\bm{a}_t$; if $\bm{a}_t \neq \bm{a}$, set $\mathcal{D} \gets \mathcal{D} \setminus \{\bm{a}\}$
        \State Update $I_{\bm{a}_t}^{i}$ with radius~\eqref{eq:eps-B-conf}
    \EndFor
\EndWhile
\State \Return $\mathcal{D}$
\end{algorithmic}
\end{algorithm}

\begin{algorithm}[H]
\caption{\texttt{Uniform-Intervals} for Problem C (Fixed Confidence)---single agent $i$}
\label{alg:unif-C-confidence}
\begin{algorithmic}[1]
\Require joint-arm ordering (pre-agreed), target $K$, confidence $\delta$
\State $I_{\bm{a}}^{i} \gets (-\infty,\infty)$ for all $\bm{a}$; $\texttt{ready}_i \gets \textsc{false}$; $\widehat S_i \gets \varnothing$
\For{$t = 1,2,\dots$}
    \State Let $\bm{a}$ be the next joint arm in the pre-agreed round-robin ordering
    \State \textbf{always} pull component $\bm{a}[i]$ \Comment{never deviate: keeps every other stream clean}
    \State Observe own reward; update $\widehat{\mu}_{\bm{a}}^{i}$, $n_{\bm{a}}^{i}(t)$, and $I_{\bm{a}}^{i}$ with radius~\eqref{eq:eps-C-conf}
    \If{$\texttt{ready}_i = \textsc{false}$ \textbf{and} $\exists\,\mathcal{S}\!\subseteq\!\mathcal{A},\,|\mathcal{S}|\!=\!K : \forall \bm{b}\!\in\!\mathcal{S},\,\bm{d}\!\in\!\mathcal{A}\!\setminus\!\mathcal{S},\;\min I_{\bm{b}}^{i} > \max I_{\bm{d}}^{i}$}
        \State $\widehat S_i \gets \mathcal{S}$; $\texttt{ready}_i \gets \textsc{true}$ \Comment{commit answer, but do \emph{not} leave}
    \EndIf
    \State \textbf{halt-check:} stop sampling once $\texttt{ready}_i$ and the safe-exit test of Remark~\ref{rem:C-stopping} holds
\EndFor
\State \Return $\widehat S_i$
\end{algorithmic}
\end{algorithm}

\section{Proofs}
\label{apx:proofs}

\subsection{Proof of Lemma~\ref{lem:num_pulls}}
\label{apx:proof-lemnumpulls}

Let $s$ denote the number of times a suboptimal joint action $\bm a$ is executed before elimination. Before elimination, each remaining action in $\mathcal D$ is executed equally often, so under the good event a top-$K$ action $\bm a^*$ is also executed $s$ times. For $I_{\bm a}(t)$ to become disjoint from (and fall below) the interval of $\sigma(K)$, it suffices that
\[
4\sqrt{\alpha/s} \;\leq\; \Delta_{\bm a}^{\langle K\rangle},
\]
giving $s \ge 16\alpha/(\Delta_{\bm a}^{\langle K\rangle})^2$. Summing over all suboptimal $\bm a$:
$T \ge \sum_{\bm a} 16\alpha/(\Delta_{\bm a}^{\langle K\rangle})^2 = 16H_1\alpha$,
hence $\alpha \le T/(16H_1)$. \hfill$\square$

\subsection{Proof of Theorem~\ref{thm:algoA-budget}}
\label{apx:proof-thmAbudget}

Define the good event $G = \bigcap_{\bm a, t\le T}\{\mu_{\bm a}\in I_{\bm a}(t)\}$. Under $G$, no top-$K$ action is eliminated, so $\mathbb{P}(\text{error})\le\mathbb{P}(G^c)$. By Lemma~\ref{corollary5.5} and a union bound:
\[
\mathbb{P}(G^c) \le \sum_{\bm a}\sum_{t=1}^T 2\exp(-\alpha/2) = 2A^M T\exp(-\alpha/2). \hfill\square
\]

\subsection{Proof of Theorem~\ref{thm:algoB-budget}}
\label{apx:proof-thmBbudget}

Index the good event by player: $G = \bigcap_{i,\bm a,t\le T}\{\mu_{\bm a}\in I_{\bm a}^i(t)\}$. A union bound over $(i,\bm a,t)$ introduces an extra factor $M$:
$\mathbb{P}(G^c) \le 2MA^MT\exp(-\alpha/2)$.
The regret bound follows. \hfill$\square$

\subsection{Proof of Theorem~\ref{thm:algoC-budget}}
\label{apx:proof-thmCbudget}

For player $i$, let $G_{\bm{a}}^i = \{|\widehat{\mu}_{\bm{a}}^i(T)-\mu_{\bm{a}}| < \tfrac12\Delta_{\bm{a}}^{\langle K\rangle}\}$.
Under $\bigcap_{\bm a,i}G_{\bm a}^i$, for any top-$K$ arm $\bm a^*$ and non-top-$K$ arm $\bm a$:
\[
\widehat{\mu}_{\bm{a^*}}(T) > \mu_{\bm{a^*}} - \tfrac12\Delta_{\bm{a^*}}^{\langle K\rangle}
= \tfrac{\mu_{\bm{a^*}}+\mu_{\sigma(K+1)}}{2}
\ge \tfrac{\mu_{\sigma(K)}+\mu_{\bm a}}{2}
= \mu_{\bm a}+\tfrac12\Delta_{\bm a}^{\langle K\rangle}
> \widehat{\mu}_{\bm a}(T),
\]
so all players identify the correct top-$K$ set. By Hoeffding and a union bound over $(i,\bm a)$:
$\mathbb{P}((\bigcap G_{\bm a}^i)^c) \le M\sum_{\bm a}2\exp(-T(\Delta_{\bm a}^{\langle K\rangle})^2/(8A^M))$. \hfill$\square$

\subsection{Proof of Theorem~\ref{thm:algoA-confidence}}
\label{apx:proof-thmAconf}

By Hoeffding's inequality with $\epsilon_{\bm{a}}(t) = \sqrt{2\log(\pi^2 A^M t^2/(3\delta))/n_{\bm{a}}(t)}$:
$\mathbb{P}(|\widehat{\mu}_{\bm{a}}(t)-\mu_{\bm{a}}|>\epsilon_{\bm{a}}(t)) \le 6\delta/(\pi^2 A^M t^2)$.
Defining $G = \bigcap_{\bm a, t\ge 1}\{|\widehat{\mu}_{\bm a}(t)-\mu_{\bm a}|\le\epsilon_{\bm a}(t)\}$ and taking a union bound:
$\mathbb{P}(G^c) \le \sum_{t\ge 1}\sum_{\bm a} 6\delta/(\pi^2 A^M t^2) = \delta$.
Under $G$ no top-$K$ arm is eliminated, giving correctness. For termination, a suboptimal $\bm a$ is eliminated once $4\epsilon_{\bm a}(t)\le\Delta_{\bm a}^{\langle K\rangle}$, requiring $n_{\bm a}(t)\ge 32\log(\pi^2 A^M t^2/(3\delta))/(\Delta_{\bm a}^{\langle K\rangle})^2$. Summing over all eliminated arms gives~\eqref{eq:tau-A}. \hfill$\square$

\subsection{Proof of Theorem~\ref{thm:algoB-confidence}}
\label{apx:proof-thmBconf}

With radius~\eqref{eq:eps-B-conf}, Hoeffding gives $\mathbb{P}(|\widehat{\mu}_{\bm a}^i(t)-\mu_{\bm a}|>\epsilon_{\bm a}^i(t))\le 6\delta/(\pi^2 A^M M t^2)$. A union bound over $(i,\bm a,t)$:
$\mathbb{P}(G^c)\le\sum_{t,i,\bm a}6\delta/(\pi^2 A^M M t^2)=\delta$.
Correctness and termination bound~\eqref{eq:tau-B} follow by the same argument as Theorem~\ref{thm:algoA-confidence} with an extra $M$ in the logarithm. \hfill$\square$

\subsection{Proof of Theorem~\ref{thm:algoC-confidence}}
\label{apx:proof-thmCconf}

Define the good event $G$ as in Theorem~\ref{thm:algoB-confidence} using the inner radius $\sqrt{2\log(\pi^2 A^M M t^2/(3\delta))/n_{\bm a}^i(t)}$; we have $\mathbb{P}(G)\ge 1-\delta$ by the same union bound. Under $G$, for any two players $i,j$ and arm $\bm a$:
\[
|\widehat{\mu}_{\bm a}^i(t) - \widehat{\mu}_{\bm a}^j(t)|
\le |\widehat{\mu}_{\bm a}^i(t)-\mu_{\bm a}| + |\widehat{\mu}_{\bm a}^j(t)-\mu_{\bm a}|
\le 2\sqrt{\frac{2\log(\pi^2 A^M M t^2/(3\delta))}{n_{\bm a}^i(t)}}
= \epsilon_{\bm a}^i(t).
\]
We now argue that certification pins down the \emph{true} top-$K$ set, which is what makes agreement automatic without any cross-player comparison. Suppose player $i$ certifies $\mathcal S$ at round $t$, i.e.\ $\min I^i_{\bm b}(t) > \max I^i_{\bm d}(t)$ for all $\bm b\in\mathcal S$, $\bm d\notin\mathcal S$, using the enlarged radius $\epsilon^i_{\bm a}(t)=2\beta_t$ with $\beta_t:=\sqrt{2\log(\pi^2 A^M M t^2/(3\delta))/n^i_{\bm a}(t)}$ the base half-width. Certification gives $\widehat\mu^i_{\bm b}-\widehat\mu^i_{\bm d} > 2\epsilon^i_{\bm a}(t)=4\beta_t$, and under $G$ each empirical mean lies within $\beta_t$ of the truth, so
\[
\mu_{\bm b}-\mu_{\bm d} \;\ge\; \big(\widehat\mu^i_{\bm b}-\widehat\mu^i_{\bm d}\big)-2\beta_t \;>\; 4\beta_t-2\beta_t \;=\;2\beta_t \;>\;0
\quad\text{for all } \bm b\in\mathcal S,\ \bm d\notin\mathcal S,
\]
so $\mathcal S$ is exactly the true top-$K$ set $S^\star$. Since the schedule and radius are identical functions of $t$ across players, this holds for whichever player certifies and at whatever (possibly different) round it does; every committed player therefore outputs $S^\star$, and the outputs coincide even though certification times $\tau_i$ differ. Because a committed player keeps executing its scheduled component (Algorithm~\ref{alg:unif-C-confidence}, line~4), no teammate's stream is perturbed, and by round $\tau=\max_i\tau_i$ all players hold $S^\star$; the safe-exit tests of Remark~\ref{rem:C-stopping} govern physical termination. Termination bound~\eqref{eq:tau-C} follows as in Theorem~\ref{thm:algoA-confidence}: certifying $4\beta_t\le\Delta^{\langle K\rangle}_{\bm a}$ at the boundary requires $n^i_{\bm a}(t)\ge 8\cdot 2\log(\pi^2A^MM t^2/(3\delta))/(\Delta^{\langle K\rangle}_{\bm a})^2$, i.e.\ the factor-$2$ radius enlargement increases required pulls by $4\times$, giving the factor $8$. \hfill$\square$

\subsection{Proof of Theorem~\ref{thm:unified} (Unified Guarantee)}
\label{apx:proof-unified}

Fix the radius~\eqref{eq:unified-radius}. We first establish $\delta$-correctness. By Hoeffding's inequality, for each player $i$, arm $\bm a$, and round $t$,
\[
\mathbb{P}\!\left(|\widehat{\mu}_{\bm a}^i(t) - \mu_{\bm a}| > \tfrac{1}{\rho}\epsilon_{\bm a}^i(t)\right)
\le 2\exp\!\left(-\frac{n_{\bm a}^i(t)}{2}\cdot\frac{2\log(c\pi^2 A^M t^2/(3\delta))}{n_{\bm a}^i(t)}\right)
= \frac{6\delta}{c\,\pi^2 A^M t^2}.
\]
Let $G$ be the event that $|\widehat{\mu}_{\bm a}^i(t) - \mu_{\bm a}| \le \tfrac{1}{\rho}\epsilon_{\bm a}^i(t)$ for all $i, \bm a, t$. There are $c$ independent player-streams to control ($c=1$ if rewards are shared, $c=M$ otherwise), so a union bound over the $c \cdot A^M$ relevant (player, arm) pairs and over $t \ge 1$ gives
\[
\mathbb{P}(G^c) \le \sum_{t\ge 1}\, c \cdot A^M \cdot \frac{6\delta}{c\,\pi^2 A^M t^2}
= \frac{6\delta}{\pi^2}\sum_{t\ge 1}\frac{1}{t^2} = \delta.
\]
Under $G$, every confidence interval (inflated by $\rho$) contains the true mean, so no top-$K$ arm is ever eliminated and, at termination, the certified separation is correct; this gives correctness for $\rho = 1$. For $\rho = 2$ (Problem C, no synchronization signal), the doubled radius additionally guarantees consensus: for any two players $i, j$,
$|\widehat{\mu}_{\bm a}^i(t) - \widehat{\mu}_{\bm a}^j(t)| \le \tfrac{2}{\rho}\epsilon_{\bm a}^i(t) = \epsilon_{\bm a}^i(t)$ under $G$, so any separation certified by one player is consistent with every other player's estimates.

For the stopping time, a sub-optimal arm $\bm a$ is eliminated once its $\rho$-inflated interval separates from the $K$-th candidate, which holds as soon as $4\epsilon_{\bm a}^i(t) \le \Delta_{\bm a}^{\langle K\rangle}$, i.e.
\[
n_{\bm a}^i(t) \;\ge\; \frac{16\rho^2 \cdot 2\log(c\pi^2 A^M t^2/(3\delta))}{(\Delta_{\bm a}^{\langle K\rangle})^2}\cdot\frac{1}{8}
\;=\; \frac{2\rho^2\log(c\pi^2 A^M t^2/(3\delta))}{(\Delta_{\bm a}^{\langle K\rangle})^2}.
\]
Summing over all arms and using $\tau = \sum_{\bm a} n_{\bm a}(\tau)$ gives~\eqref{eq:unified-tau}. Setting $(c,\rho)=(1,1),(M,1),(M,2)$ reproduces Theorems~\ref{thm:algoA-confidence}--\ref{thm:algoC-confidence}. \hfill$\square$

\subsection{Proof of Lemma~\ref{lem:transport} (Transportation Inequality)}
\label{apx:proof-transport}

This is the standard change-of-measure argument \cite{kaufmann2016complexity}. Let $\mathcal{F}_\tau$ be the history up to the stopping time $\tau$. By the data-processing inequality applied to the (possibly randomized) decision $\widehat{S} \in \{S^\star(\nu), \text{not}\}$, the log-likelihood ratio between $\nu$ and $\nu'$ satisfies
\[
\mathbb{E}_\nu\!\left[\log\frac{d\mathbb{P}_\nu}{d\mathbb{P}_{\nu'}}(\mathcal{F}_\tau)\right]
\;\ge\;
\mathrm{kl}\!\bigl(\mathbb{P}_\nu(\widehat{S}=S^\star(\nu)),\,\mathbb{P}_{\nu'}(\widehat{S}=S^\star(\nu))\bigr).
\]
The left side equals $\sum_{\bm a}\mathbb{E}_\nu[N_{\bm a}(\tau)]\,\mathrm{KL}(\nu_{\bm a},\nu'_{\bm a})$ by Wald's identity, since the reward observed when arm $\bm a$ is pulled has KL divergence $\mathrm{KL}(\nu_{\bm a},\nu'_{\bm a})$ between the two models. Because $\mathcal{B}$ is $\delta$-correct and $S^\star(\nu')\neq S^\star(\nu)$, we have $\mathbb{P}_\nu(\widehat{S}=S^\star(\nu)) \ge 1-\delta$ and $\mathbb{P}_{\nu'}(\widehat{S}=S^\star(\nu)) \le \delta$, so the right side is at least $\mathrm{kl}(1-\delta,\delta) = \mathrm{kl}(\delta,1-\delta) \ge \log\frac{1}{2.4\delta}$. \hfill$\square$

\subsection{Proof details for Theorem~\ref{thm:lower-fc} (Lower Bound)}
\label{apx:proof-lower}

We give the construction omitted from the main text. Index the arms by rank via $\sigma$. For a sub-optimal arm $\bm a$ (rank $> K$), let $\nu^{(\bm a)}$ agree with $\nu$ on every coordinate except $\mu_{\bm a}$, which is raised to $\mu_{\sigma(K)} + \eta$ for arbitrarily small $\eta > 0$; the required perturbation magnitude is $\mu_{\sigma(K)} - \mu_{\bm a} = \Delta_{\bm a}^{\langle K\rangle}$. In $\nu^{(\bm a)}$, arm $\bm a$ now ranks within the top $K$ while $\sigma(K)$ drops out, so $S^\star(\nu^{(\bm a)}) \neq S^\star(\nu)$. For a top-$K$ arm $\bm a$ (rank $\le K$), symmetrically lower $\mu_{\bm a}$ to $\mu_{\sigma(K+1)} - \eta$, a perturbation of $\mu_{\bm a} - \mu_{\sigma(K+1)} = \Delta_{\bm a}^{\langle K\rangle}$. Applying Lemma~\ref{lem:transport} to the pair $(\nu, \nu^{(\bm a)})$, in which only coordinate $\bm a$ changes,
\[
\mathbb{E}_\nu[N_{\bm a}(\tau)]\cdot \frac{(\Delta_{\bm a}^{\langle K\rangle})^2}{2}
\;\ge\; \log\tfrac{1}{2.4\delta}
\quad\Longrightarrow\quad
\mathbb{E}_\nu[N_{\bm a}(\tau)] \ge \frac{2\log\tfrac{1}{2.4\delta}}{(\Delta_{\bm a}^{\langle K\rangle})^2}.
\]
Summing over all $A^M$ arms gives $\mathbb{E}_\nu[\tau] \ge 2 H_1\log\frac{1}{2.4\delta}$.

\emph{Extension to Problems B and C.} In these regimes each player $i$ observes only its own reward stream. Restricting attention to player $i$'s decision $\widehat{S}_i$, which must equal $S^\star(\nu)$ with probability $\ge 1-\delta$, the same construction and transportation inequality applied to player $i$'s observation process yield $\mathbb{E}_\nu[N_{\bm a}^{i}(\tau)] \ge 2\log\frac{1}{2.4\delta}/(\Delta_{\bm a}^{\langle K\rangle})^2$. Since in the exploration phase player $i$ observes one reward per round for the prescribed joint action, $N_{\bm a}^{i}(\tau) = N_{\bm a}(\tau)$, and summing reproduces~\eqref{eq:lower-fc}. \hfill$\square$

\section{Experimental Details}
\label{apx:experiments}

We collect the details needed to reproduce Figures~\ref{fig:main} and~\ref{fig:scale}.

\paragraph{Data-generating processes.}
Two families of instances are used. The \emph{synthetic Gaussian} instance of Section~\ref{sec:experiments} has $M=2$, $A=3$, $K=3$, mean vector $\bm\mu=[0.80,0.75,0.72,0.50,0.45,0.40,0.35,0.30,0.25]$ (rank order), unit-variance Gaussian noise, boundary gap $\Delta=0.22$, and $H_1\approx 109$. The \emph{CTR-style} instances use Bernoulli (click) rewards: at each pull of arm $\bm a$ the observed reward is $\mathrm{Bernoulli}(\mu_{\bm a})$ with $\mu_{\bm a}\in(0,1)$ the arm's click-through rate. The fixed-budget panel (Fig.~\ref{fig:main}a) uses $N=50$ joint arms with a top-$K$ boundary gap of $0.06$ and $K=3$; the fixed-confidence panel (Fig.~\ref{fig:main}b) uses a smaller instance with $N=24$ and boundary gap $0.15$. In every case the click rates are laid out in decreasing rank order with the prescribed boundary gap at the $K/K{+}1$ interface and the remaining arms spread below it; Bernoulli noise is $1/2$-subgaussian, so all analysis constants apply verbatim.

\paragraph{Trials and error bars.}
Each plotted point is the mean over independent trials with fresh reward draws: $80$ trials per budget in Fig.~\ref{fig:main}a; $12$ trials for Problems~A/B and $4$ for Problem~C in Fig.~\ref{fig:main}b (Problem~C is limited by the cost of its long uniform runs); and the fixed-budget scalability sweep of Fig.~\ref{fig:scale} uses fixed total budget $T=24{,}000$ with $A=3$ and $M\in\{2,\dots,5\}$ ($N=9,\dots,243$). Reported variability is the standard error of the mismatch indicator across trials; because Problem~C uses few trials, its fixed-confidence curve should be read as indicative rather than tightly estimated.

\paragraph{Stopping and output criteria.}
In the fixed-budget experiments each method runs for exactly $T$ rounds and then outputs its current candidate set (the surviving set $\mathcal D$ for the elimination methods, or the empirical top-$K$ by $\widehat\mu^i$ for uniform exploration); an error is recorded whenever the output differs from $S^\star$. In the fixed-confidence experiments a run terminates at the first round the certification test fires (Appendix~\ref{apx:fc-algos}), and $\tau$ is that round; for Problem~C we record each player's private certification time and report $\tau=\max_i\tau_i$, consistent with Remark~\ref{rem:C-stopping}.

\paragraph{Choice of $\alpha$ and sensitivity.}
The fixed-budget exploration constant $\alpha=T/(16\widehat H)$ requires a complexity estimate $\widehat H$. When $\widehat H$ overestimates $H_1$ the intervals are conservatively wide, slowing elimination but never harming correctness on the good event; when $\widehat H$ underestimates $H_1$ the budget may run out before the hardest arms separate. We therefore recommend the doubling schedule of Remark~\ref{rem:H1-knowledge}: sweep $\widehat H\in\{2,4,8,\dots\}$ and keep the first that certifies, which costs at most a constant factor in budget and removes the need to know $H_1$. The qualitative ordering of the three methods in Figures~\ref{fig:main}--\ref{fig:scale} is insensitive to $\widehat H$ within a factor of two of $H_1$.

\section{$(\varepsilon,\delta)$-PAC Relaxation}
\label{apx:pac}

\begin{proposition}[$(\varepsilon,\delta)$-PAC guarantee]
\label{prop:pac}
Suppose we only require identifying a set $\widehat{S}$ whose $K$-th largest mean is within $\varepsilon$ of $\mu_{\sigma(K)}$ (an $\varepsilon$-good top-$K$ set). Then all algorithms of Sections~\ref{sec:probA}--\ref{sec:probC} apply verbatim with each gap $\Delta_{\bm{a}}^{\langle K\rangle}$ replaced by $\max\{\Delta_{\bm{a}}^{\langle K\rangle},\,\varepsilon\}$, and Theorem~\ref{thm:unified} holds with $H_1$ replaced by $H_1^{(\varepsilon)} = \sum_{\bm{a}} \max\{\Delta_{\bm{a}}^{\langle K\rangle}, \varepsilon\}^{-2} \le A^M/\varepsilon^2$. In particular, the stopping time becomes \emph{instance-independent}, bounded by $O(\rho^2 (A^M/\varepsilon^2)\log(cA^M/\delta))$, removing the dependence on the smallest gap.
\end{proposition}
\begin{proof}
The good-event and elimination arguments of the proofs above are unchanged once each $\Delta_{\bm{a}}^{\langle K\rangle}$ is floored at $\varepsilon$: any arm whose true gap is below $\varepsilon$ need not be separated from the boundary (declaring it either in or out of $\widehat S$ yields an $\varepsilon$-good set), so such arms never become the sampling bottleneck. Each elimination now requires $4\epsilon_{\bm a}^i(t) \le \max\{\Delta_{\bm a}^{\langle K\rangle},\varepsilon\}$, giving the stated per-arm sample counts; summing yields $H_1^{(\varepsilon)}$. The bound $H_1^{(\varepsilon)} \le A^M/\varepsilon^2$ holds since each of the $A^M$ terms is at most $\varepsilon^{-2}$. \hfill$\square$
\end{proof}

\end{document}